\documentclass[11pt, hidelinks]{article} 

\usepackage{hyperref}
\usepackage{url}
\usepackage{mathtools}
\usepackage{fullpage}

\usepackage{amsmath,amssymb,amsthm, amsfonts}
\usepackage{mathbbol}
\usepackage{mathrsfs}  
\usepackage{multicol}
\usepackage{rotating}
\usepackage{enumitem, comment, xifthen}

\usepackage{xcolor}
\usepackage{mathbbol}

\usepackage{booktabs}

\newcommand{\figdir}{.}
\newcommand{\myqvar}[1]{\ensuremath{\sigma(#1)}}
\newcommand{\myqvarsq}[1]{\ensuremath{\sigma^2(#1)}}

\theoremstyle{plain}
\newtheorem{theorem}{Theorem}
\newtheorem{lemma}{Lemma}
\newtheorem{corollary}{Corollary}

\theoremstyle{definition}

\newtheorem{example}{Example}

\theoremstyle{remark}

\newcommand{\xstate}[1]{\ensuremath{x_{#1}}}

\newcommand{\myoperator}[1]{\mathcal{#1}}

\newcommand{\detit}[1]{\ITCOMMAND{b}{#1}}

\newcommand{\ITCOMMAND}[2]{\ensuremath{#1}_{#2}}
\newcommand{\pathit}[1]{\ITCOMMAND{P}{#1}}

\newcommand{\avarit}[1]{\ITCOMMAND{a}{#1}}

\newcommand{\zerovec}{\ensuremath{\mathbb{0}}}

\newcommand{\thetait}[1]{\ITCOMMAND{\theta}{#1}}
\newcommand{\DeltaIt}[1]{\ITCOMMAND{\Delta}{#1}}
\newcommand{\DelIt}[1]{\DeltaIt{#1}}
\newcommand{\wnoiseit}[1]{\ITCOMMAND{W}{#1}}
\newcommand{\epsnoiseit}[1]{\ITCOMMAND{E}{#1}}

\newcommand{\Hop}{\ensuremath{\mathcal{H}}}
\newcommand{\HopIter}[1]{\ensuremath{\ITCOMMAND}{\Hop}{#1}}
\newcommand{\TopIter}[1]{\ensuremath{\ITCOMMAND}{\Top}{#1}}

\newcommand{\titer}{\ensuremath{k}}

\newcommand{\state}{\ensuremath{x}}
\newcommand{\action}{\ensuremath{u}}

\newcommand{\reward}{\ensuremath{r}}

\newcommand{\TfunPlain}{\ensuremath{\mathbb{P}}}
\newcommand{\Trans}[3]{\ensuremath{\TfunPlain_{#2}(#3 \mid #1)}}
\newcommand{\Tfun}{\Trans}

\newcommand{\discount}{\ensuremath{\gamma}}

\newcommand{\Bellman}{\ensuremath{\myoperator{T}}}

\newcommand{\EmpBellman}{\ensuremath{\widehat{\Bellman}}}
\newcommand{\EmpBellmanIter}[1]{\ITCOMMAND{\EmpBellman}{#1}}

\newcommand{\pol}{\ensuremath{\pi}}

\newcommand{\ActionSpace}{\ensuremath{\mathcal{U}}}
\newcommand{\StateSpace}{\ensuremath{\mathcal{X}}}

\newcommand{\iter}{\ensuremath{\ell}}

\newcommand{\stepit}[1]{\ensuremath{\stepcon_{#1}}}

\newcommand{\stepcon}{\ensuremath{\lambda}}

\newcommand{\myint}{\ensuremath{\mbox{int}}}

\newcommand{\bcar}{\begin{itemize}}
\newcommand{\ecar}{\end{itemize}}

\newcommand{\LinSpace}{\ensuremath{\mathbb{L}}}

\newcommand{\thetastar}{\ensuremath{{\theta^*}}}

\newcommand{\opnorm}[1]{\ensuremath{\matsnorm{#1}{2}}}

\newcommand{\vsmall}{\vspace*{.1in}}

\definecolor{MyGray}{rgb}{0.9,0.9,0.9}

\makeatletter\newenvironment{graybox}{ 
\begin{lrbox}{\@tempboxa}
\begin{minipage}{0.985\columnwidth}}{\end{minipage}
\end{lrbox}%
\colorbox{MyGray}{\usebox{\@tempboxa}} }
\makeatother



\newcommand{\Zback}[1]{\ensuremath{Z^{\backslash i}}}

\newcommand{\xsamstack}[1]{\ensuremath{x_1^\numobs}}
\newcommand{\Xsamstack}[1]{\ensuremath{X_1^\numobs}}

\newcommand{\widgraph}[2]{\includegraphics[keepaspectratio,width=#1]{#2}}

\newcommand{\Term}{\ensuremath{T}}

\newcommand{\fancysoln}[1]{
\ifthenelse{\equal{\doctype}{WITHSOLS}}
{
\begin{soln}
#1
\end{soln}
}
{
}
}

\newcommand{\goodendex}{\hfill $\clubsuit$}

\long\def\comment#1{}

\makeatletter
\def\@cite#1#2{[\if@tempswa #2 \fi #1]}
\makeatother

\newcommand{\defn}{\vcentcolon=}

\newcommand{\var}{\ensuremath{\operatorname{var}}}

\newcommand{\real}{\ensuremath{\mathbb{R}}}

\newcommand{\numobs}{\ensuremath{n}}

\newlength{\widebarargwidth}
\newlength{\widebarargheight}
\newlength{\widebarargdepth}

\newcommand{\mns}{\mkern-1mu}  
\newcommand{\matsnorm}[2]{|\mns|\mns| #1 |\mns|\mns|_{{#2}}}

\newcommand{\Exs}{\ensuremath{\mathbb{E}}}

\makeatletter
\long\def\@makecaption#1#2{
        \vskip 0.8ex
        \setbox\@tempboxa\hbox{\small {\bf #1:} #2}
        \parindent 1.5em  
        \dimen0=\hsize
        \advance\dimen0 by -3em
        \ifdim \wd\@tempboxa >\dimen0
                \hbox to \hsize{
                        \parindent 0em
                        \hfil 
                        \parbox{\dimen0}{\def\baselinestretch{0.96}\small
                                {\bf #1.} #2
                                } 
                        \hfil}
        \else \hbox to \hsize{\hfil \box\@tempboxa \hfil}
        \fi
        }
\makeatother

\newcommand{\spnorm}[1]{\ensuremath{\|#1\|_{\mbox{\tiny{sp}}}}}

\newcommand{\Top}{\ensuremath{\mathcal{G}}}
\newcommand{\idvec}{\ensuremath{\mathbf{e}}}
\newcommand{\Cone}{\ensuremath{\mathbb{K}}}
\newcommand{\coneleq}{\ensuremath{\preceq}}
\newcommand{\idnorm}[1]{\ensuremath{\|#1\|_{\idvec}}}
\renewcommand{\LinSpace}{\ensuremath{\mathbb{V}}}

\newcommand{\unicon}{\ensuremath{c}}
\newcommand{\uniom}{\ensuremath{\unicon_\omega}}
\newcommand{\TermB}{\ensuremath{\Term'}}
\newcommand{\unicontwo}{\ensuremath{c'}}
\newcommand{\rmax}{\ensuremath{r_{\mbox{\tiny{max}}}}}

\newcommand{\hanabou}{\ensuremath{B}}

\newcommand{\slam}{\ensuremath{s}}

\newcommand{\nepsnoiseit}[1]{\ensuremath{\varepsilon_{#1}}}
\newcommand{\Vio}{\ensuremath{V}}

\newcommand{\ConeOrth}{\ensuremath{\Cone_{\mbox{\tiny{orth}}}}}
\newcommand{\ConePSD}{\ensuremath{\Cone_{\mbox{\tiny{PSD}}}}}

\newcommand{\Qpol}{\ensuremath{\theta^\pol}}

\newcommand{\TLR}{\ensuremath{T_{\mbox{\tiny{LinRes}}}}}

\newcommand{\IdMat}{\ensuremath{\mathbf{I}}}

\newcommand{\spannorm}[1]{\ensuremath{\|#1\|_{\mbox{\tiny{span}}}}}

\newcommand{\MDP}{\ensuremath{\mathscr{M}}}

\newcommand{\betalin}{\ensuremath{\widehat{\beta}_{\mbox{\tiny{lin}}}}}
\newcommand{\betapoly}{\ensuremath{\widehat{\beta}_{\mbox{\tiny{poly}}}}}

\newcommand{\contract}{\ensuremath{\nu}}
\newcommand{\contiter}[1]{\ensuremath{\contract_{#1}}}

\begin{document}

\begin{center}
  {\bf{\LARGE{Stochastic approximation with cone-contractive operators:
        Sharper $\ell_\infty$-bounds for $Q$-learning}}}

  \vspace*{0.5in}
  
  \begin{tabular}{c}
    Martin J. Wainwright  (GD) \\
    Departments of Statistics and EECS \\
    UC Berkeley \\
    Voleon Group, Berkeley, CA \\
    \texttt{wainwrig@berkeley.edu}
  \end{tabular}
  
  \vsmall

  \begin{abstract}
    Motivated by the study of $Q$-learning algorithms in reinforcement
    learning, we study a class of stochastic approximation procedures
    based on operators that satisfy monotonicity and
    quasi-contractivity conditions with respect to an underlying cone.
    We prove a general sandwich relation on the iterate error at each
    time, and use it to derive non-asymptotic bounds on the error in
    terms of a cone-induced gauge norm.  These results are derived
    within a deterministic framework, requiring no assumptions on the
    noise. We illustrate these general bounds in application to
    synchronous $Q$-learning for discounted Markov decision processes
    with discrete state-action spaces, in particular by deriving
    non-asymptotic bounds on the $\ell_\infty$-norm for a range of
    stepsizes.  These results are the sharpest known to date, and we
    show via simulation that the dependence of our bounds cannot be
    improved in any uniform way.  These results show that relative to
    model-based $Q$-iteration, the $\ell_\infty$-based sample
    complexity of $Q$-learning is suboptimal in terms of the discount
    factor $\gamma$.
  \end{abstract}

\end{center}


\section{Introduction}

Stochastic approximation (SA) algorithms are widely used in many
areas, including stochastic control, communications, machine learning,
statistical signal processing and reinforcement learning, among
others.  There is now a very rich literature on SA algorithms, their
applications and the associated theory (e.g., see the
books~\cite{Benveniste90, Kushner97, Bor08} and references therein).
One set of fundamental questions concerns the convergence of SA
algorithms; there are various general techniques for establishing
convergence, including the ODE method, dynamical system, and
Lyapunov-based methods, among others.  Much of the classical theory in
stochastic approximation is asymptotic in nature, whereas in more
recent work, particularly in the special case of stochastic
optimization, attention has been shifted to non-asymptotic
results~\cite{NemJudLanSha09, BacMou11}.

The goal of this paper is to develop some non-asymptotic bounds for a
certain class of stochastic approximation procedures.  The motivating
impetus for this work was to gain a deeper insight into the classical
$Q$-learning algorithm~\cite{WatDay92} from Markov decision processes
and reinforcement learning~\cite{Puterman05, SutBar18, Bertsekas_dyn1,
  BerTsi96, Sze09}.  It is a stochastic approximation algorithm for
solving a fixed point equation involving the Bellman operator. In the
discounted setting, this operator is contractive with respect to a sup
norm, and also monotonic in the elementwise ordering.  We show that
these conditions can be viewed as special cases of a more general
structure on the operators used in stochastic approximation for
solving fixed point equations.  In particular, we introduce
monotonicity and quasi-contractivity conditions that are defined with
respect to the partial order and gauge norms induced by an underlying
cone.  In the case of sup norm contractions, this underlying cone is
the orthant cone, but other cones also arise naturally in
applications.  For instance, for SA procedures that operate in the
space of symmetric matrices, the cone of positive semidefinite
matrices induces the spectral order, as well as various forms of
spectral norms.  For a sequence of operators satisfying these cone
monotonicity and quasi-contractivity conditions, we prove a general
result (Theorem~\ref{ThmConeBound}) that sandwiches the error at each
iteration in terms of the partial order induced by the cone.  By
considering concrete choices of stepsize---such as linearly or
polynomial decaying ones---we derive corollaries that yield
non-asymptotic bounds on the error.

We specialize this general theory to the synchronous form of
$Q$-learning in discounted Markov decision processes, and use it to
derive non-asymptotic bounds on the $\ell_\infty$-error of
$Q$-learning, for both polynomial stepsizes and a linearly rescaled
stepsize.  Notably, these results are the sharpest known to date, and
depend on the structure of the optimal $Q$-function.  We show via
simulation studies that our bounds are unimprovable in general.  In
particular, we exhibit a ``hard'' problem for which our theory
predicts that the number of iterations required to obtain an
$\epsilon$-accurate solution in $\ell_\infty$-norm should scale as
$\frac{1}{(1-\discount)^4} \frac{1}{\epsilon^2}$, and show that this
prediction is empirically sharp.  In the the worst case setting, our
bounds lead to $\frac{1}{(1-\discount)^5}$ scaling, a scaling that
matches known bounds on synchronous $Q$-learning from previous
work~\cite{EveMan03}; however, we are not aware of a problem for which
this worst-case bound is actually sharp.  For context, we note that
the speedy-Q-learning method, an extension of ordinary $Q$-learning,
is known to have iteration complexity scaling as $\frac{1}{(1 -
  \discount)^4} \frac{1}{\epsilon^2}$.  Moreover, Azar et
al.~\cite{AzaMunKap13} show that model-based $Q$-iteration exhibits a
$\frac{1}{(1-\discount)^3} \frac{1}{\epsilon^2}$ scaling, and moreover
that this is the best possible for any method in a minimax sense.
Consequently, a corollary of our results is to reveal a gap between
the performance of standard synchronous $Q$-learning and an optimal
(model-based) procedure.

The remainder of this paper is organized as follows.  In
Section~\ref{SecGeneral}, we introduce the class of stochastic
approximation algorithms analyzed in this paper, including some
required background on cones and induced gauge norms.  We then state
our main result (Theorem~\ref{ThmConeBound}), as well some of its
corollaries for particular stepsize choices
(Corollaries~\ref{CorDetClaim} and~\ref{CorDetClaimPoly}).  In
Section~\ref{SecQlearn}, we turn to the analysis of $Q$-learning.
After introducing the necessary background in
Section~\ref{SecQlearnBack}, we then devote
Section~\ref{SecQlearnResults} to statement of our two main results on
$Q$-learning, namely $\ell_\infty$-norm bounds for a linear rescaled
stepsize (Corollary~\ref{CorQlearnLinear}) and for polynomially
decaying stepsizes (Corollary~\ref{CorQlearnPoly}).  In
Section~\ref{SecQlearnComparison}, we discuss past work on
$Q$-learning and compare our guarantees to the best previously known
non-asymptotic results.  In Section~\ref{SecHard}, we describe and
report the results of a simulation study that provides empirical
evidence for the sharpness of our bounds. We conclude with a
discussion in Section~\ref{SecDiscussion}, with more technical aspects
of our proofs deferred to the appendices.


\section{A general convergence result}
\label{SecGeneral}
In this section, we set up the stochastic approximation algorithms of
interest.  Doing so requires some background on cones, monotonic
operators on cones, and gauge norms induced by order intervals, which
we provide in Section~\ref{SecBackground}.  In
Section~\ref{SecSandwich}, we state a general result
(Theorem~\ref{ThmConeBound}) that sandwiches the iterate error using
the partial order induced by the cone.  This result holds for
arbitrary stepsizes in the interval $(0,1)$; we follow up by using
this general result to derive specific bounds that apply to stepsize
choices commonly used in practice (cf. Corollaries~\ref{CorDetClaim}
and~\ref{CorDetClaimPoly}).


\subsection{Background and problem set-up}
\label{SecBackground}

Consider a topological vector space $\LinSpace$, and an operator
$\Hop$ that maps $\LinSpace$ to itself.  Our goal is to compute a
fixed point of $\Hop$---that is, an element $\thetastar \in \LinSpace$
such that $\Hop(\thetastar) = \thetastar$---assuming that such an
element exists and is unique. In various applications, we are not able
to evaluate $\Hop$ exactly, but instead are given access to a sequence
of auxiliary operators $\{\HopIter{\titer} \}_{\titer \geq 1}$, and
permitted to compute the quantity $y_k(\theta) =
\HopIter{\titer}(\theta) + \epsnoiseit{\titer}$ for any $\theta
\in \LinSpace$.  Here $\epsnoiseit{\titer}$ denotes an error term,
allowed to be arbitrary in the analysis of this section.  In the
simplest case, we have $\HopIter{\titer} \equiv \Hop$ for all
$\titer$, but the additional generality afforded by the setup here
turns out to be useful.

Given an observation model of this type, we consider algorithms that
generate a sequence $\{\thetait{\titer} \}_{\titer \geq 1}$ according
to the recursion
\begin{align}
\label{EqnRecursion}
\thetait{\titer+1} & = (1 - \stepit{\titer}) \thetait{\titer} +
\stepit{\titer} \left \{ \HopIter{\titer}(\thetait{\titer}) +
\epsnoiseit{\titer} \right \}.
\end{align}
The stepsize parameters $\stepit{\titer}$ are assumed to belong to the
interval $(0,1)$, and should be understood as design parameters.  Our
primary goals are to specify conditions on the \emph{auxiliary
  operators} $\{ \HopIter{\titer} \}_{\titer \geq 1}$, \emph{noise
  sequence} $\{\epsnoiseit{\titer} \}_{\titer \geq 1}$, and
\emph{stepsize sequence} $\{\stepit{\titer} \}_{\titer \geq 1}$ under
which the sequence $\{\thetait{\titer} \}_{\titer \geq 1}$ converges
to $\thetastar$.  Moreover, we seek to develop tools for proving
non-asymptotic bounds on the error---i.e., guarantees that hold for
finite iterations, as opposed to in the limit as $\titer$ increases to
infinity.

Of course, convergence guarantees are not possible without imposing
assumptions on the auxiliary operators.  In this paper, motivated by
the analysis of $Q$-learning and related algorithms in reinforcement
learning, we assume that they satisfy certain properties that depend
on a cone $\Cone$ contained in $\LinSpace$.  Let us first introduce
some relevant background on cones, order intervals and induced gauge
norms.  Any cone induces a partial order on $\LinSpace$ via the
relation
\begin{align}
  \label{EqnConePartial}
\theta \coneleq \theta' \quad \iff (\theta' - \theta) \in \Cone.
\end{align}
Cones that have non-empty interiors and are topologically
normal~\cite{AliTou07,Kad11} can also be used to induce a certain
class of gauge norms as follows. For a given element $\idvec \in
\myint(\Cone)$, the associated order interval is the set
\begin{subequations}
\begin{align}
[-\idvec, \idvec] \defn \big \{ \theta \in \LinSpace \mid -\idvec
\preceq \theta \preceq \idvec \big \}
\end{align}
and it defines the Minkowski (gauge) norm given by
\begin{align}
\idnorm{\theta} & = \inf \big \{ s > 0 \mid \theta/s \in [-\idvec,
  \idvec] \big \}.
\end{align}
\end{subequations}
Let us consider some concrete examples to illustrate.

\begin{example}[Orthant cone and $\ell_\infty$-norms]
Suppose that $\LinSpace$ is the usual Euclidean space $\real^d$, and
consider the orthant cone $\ConeOrth \defn \{ \theta \in \real^d \mid
\theta_j \geq 0 \quad \mbox{for all $j \in [d]$} \}$, where $[d] \defn
\{1, 2, \ldots, d \}$.  It induces the usual elementwise
ordering---viz. $\theta' \preceq \theta$ if and only if $\theta_j \leq
\theta'_j$ for all $j \in [d]$.  Setting $\idvec$ to be the all-ones
vector, we find that
\begin{align*}
\idnorm{\theta} & = \inf \left \{ s > 0 \mid -1 \leq \theta_j/s \leq 1
\quad \mbox{for all $j \in [d]$} \right \} \; = \underbrace{\max_{j
    \in [d]} |\theta_j|}_{ \|\theta\|_\infty}.
\end{align*}
Thus, this choice of $\idvec$ induces the usual $\ell_\infty$-norm on
vectors.  Setting $\idvec$ to some other vector contained in the
interior of the orthant cone yields a weighted $\ell_\infty$-norm.
\hfill \goodendex
\end{example}

\begin{example}[Symmetric matrices  and spectral norm]
Now suppose that $\LinSpace$ is the space of $d$-dimensional symmetric
matrices $\LinSpace = \left \{ M \in \real^{d \times d} \mid M = M^T
\right \}$.  Letting $\{\gamma_j(M) \}_{j=1}^d$ denote the eigenvalues
of a matrix $M \in \LinSpace$, consider the cone of positive
semidefinite matrices
\begin{align*}
\ConePSD = \left \{ M \in \Theta \mid \gamma_j(M) \geq 0 \quad
\mbox{for all $j \in [d]$} \right \}.
\end{align*}
This cone induces the spectral ordering $M \preceq M'$ if and only if
all the eigenvalues of $M' - M$ are non-negative. Setting $\idvec$ to be
the identity matrix $\IdMat$, we have
\begin{align*}
\idnorm{M} & = \inf \left \{ s > 0 \mid -1 \leq \gamma_j(M)/s \leq 1
\quad \mbox{for all $j \in [d]$} \right \} \; = \underbrace{\max_{j
    \in [d]} |\gamma_j(M)|}_{ \opnorm{M}},
\end{align*}
so that the induced gauge norm is the spectral norm on symmetric
matrices.  \hfill \goodendex
\end{example}

In this paper, we assume that the operators $\HopIter{\titer}$ in the
recursion~\eqref{EqnRecursion} satisfy two properties:
cone-monotonicity and cone-quasi-contractivity.  More precisely, we
assume that for each $\titer = 1, 2, \ldots$, the operator
$\HopIter{\titer}$ is \emph{monotonic with respect to the cone},
meaning that
\begin{subequations}
\begin{align}
  \label{EqnMonotonicity}
\HopIter{\titer}(\theta) \coneleq \HopIter{\titer}(\theta') \qquad
\mbox{whenever $\theta \coneleq \theta'$.}
\end{align}
Moreover, we assume that for some element $\idvec \in \myint(\Cone)$,
it is \emph{cone-quasi-contractive} meaning that there is some
$\contiter{\titer} \in (0, 1)$ and some $\thetastar \in \LinSpace$
such that
\begin{align}
  \label{EqnContractivity}
\idnorm{\HopIter{\titer}(\theta) - \HopIter{\titer}(\thetastar)} &
\leq \contiter{\titer} \idnorm{\theta - \thetastar} \qquad \mbox{for
  all $\theta \in \LinSpace$.}
\end{align}
\end{subequations}
Here the terminology ``quasi'' denotes the fact that the
relation~\eqref{EqnContractivity} only need hold for a single
$\thetastar$, as opposed to in a uniform sense. Note that it is not
necessary that $\thetastar$ be a fixed point of each
$\HopIter{\titer}$.


\subsection{A sandwich result and its corollaries}
\label{SecSandwich}

With this set-up, we now turn to the analysis of the sequence
$\{\thetait{\titer} \}_{\titer \geq 1}$ generated by a recursion of
the form~\eqref{EqnRecursion}.  We first state a general ``sandwich''
result, which provides both lower and upper bounds on the error
$\thetait{\titer} - \thetastar$ in terms of the partial order induced
by the cone.  This result holds for any sequence of stepsizes
contained in the interval $(0,1)$.  By specializing this general
theorem to particular stepsize choices that are common in stochastic
approximation, we obtain non-asymptotic upper bounds on the error, as
measured in the cone-induced norm $\idnorm{\cdot}$.

Our results depend on a form of \emph{effective noise}, defined as
follows
\begin{align}
  \label{EqnEffectiveNoise}
\wnoiseit{\titer} & \defn \HopIter{\titer}(\thetastar) - \thetastar +
\epsnoiseit{\titer}.
\end{align}
Note that the effective noise at iteration $\titer$ is the sum of the
``defect'' in the operator $\HopIter{\titer}$---meaning its failure to
preserve the target $\thetastar$ as a fixed point---and the original
error term $\epsnoiseit{\titer}$ introduced in our set-up.

Our bounds involve the sequence of elements in $\LinSpace$ defined via
the recursion
\begin{subequations}
\begin{align}
  \label{EqnDefnPathit}
  \pathit{\titer} & \defn (1 - \stepit{\titer-1}) \pathit{\titer-1} +
  \stepit{\titer-1} \wnoiseit{\titer-1}, \qquad \mbox{with
    initialization $\pathit{1} = \zerovec$,}
\end{align}
where $\zerovec$ denotes the zero element in $\LinSpace$.  It also
involves the sequences of non-negative scalars
\begin{align}
\label{EqnDefnDetit}
\detit{\titer} & \defn \left(1 - (1 - \contiter{\titer-1})
\stepit{\titer-1} \right) \detit{\titer-1} \qquad \mbox{with
  initialization $\detit{1} = \idnorm{\thetait{1} - \thetastar}$, and}
\\
\label{EqnDefnAvarit}
\avarit{\titer} & \defn \left(1 - (1 - \contiter{\titer-1})
\stepit{\titer-1} \right) \avarit{\titer-1} + \discount
\stepit{\titer} \idnorm{\pathit{\titer-1}}, \quad \mbox{with
  initialization $\avarit{1} = 0$.}
\end{align}
\end{subequations}

\begin{theorem}
  \label{ThmConeBound}
  Consider a sequence of operators $\{ \HopIter{\titer} \}_{\titer
    \geq 1}$ that are monotonic~\eqref{EqnMonotonicity} and
  \mbox{$\{\contiter{\titer} \}$-quasi}-contractive~\eqref{EqnContractivity}
  with respect to a cone $\Cone$ with gauge norm
  $\idnorm{\cdot}$. Then for any sequence of stepsizes
  $\{\stepit{\titer} \}_{\titer \geq 1}$ in the interval $(0,1)$, the
  iterates $\{\thetait{\titer} \}_{\titer \geq 1}$ generated by the
  recursion~\eqref{EqnRecursion} satisfy the sandwich relation
  \begin{align}
    \label{EqnConeBound}
    -\detit{\titer} \idvec - \avarit{\titer} \idvec + \pathit{\titer}
    \; \preceq \; \thetait{\titer} - \thetastar \; \preceq \;
    \detit{\titer} \idvec + \avarit{\titer} \idvec + \pathit{\titer},
  \end{align}
where $\preceq$ denotes the partial ordering induced by the cone.
\end{theorem}
\noindent See Appendix~\ref{AppThmConeBound} for the proof.

Theorem~\ref{ThmConeBound} is a general result that applies to any
choice of stepsizes that belong to the unit interval $(0,1)$.  By
specializing the stepsize choice, we can use the sandwich
relation~\eqref{EqnConeBound} to obtain concrete bounds on the error
$\idnorm{\thetait{\titer+1} - \thetastar}$.  In doing so, we
specialize to the case $\contiter{\titer} = \contract$, so that all
the operators share the same quasi-contractivity coefficient
$\contract \in (0,1)$.

We begin by considering a sequence of stepsizes in the interval
$(0,1)$ that satisfy the bound
\begin{align}
  \label{EqnStepBound}
\big(1 - (1 - \contract) \stepit{\titer} \big) \leq
\frac{\stepit{\titer}}{\stepit{\titer-1}}.
\end{align}
Note that the usual linear stepsize $\stepit{\titer} = 1/\titer$ does
\emph{not} satisfy this bound for $\contract \in (0,1)$.  Examples of
stepsizes that do satisfy this bound are the rescaled linear stepsize
$\stepit{\titer} = \frac{1}{(1-\contract) \titer}$, valid once $\titer
\geq \frac{1}{1 - \contract}$, as well as the shifted version of
rescaled linear stepsize $\stepit{\titer} = \frac{1}{1 + (1-\contract)
  \titer}$, valid for all iterations $\titer \geq 1$.

\begin{corollary}[Bounds for linear stepsizes]
\label{CorDetClaim}
Under the assumptions of Theorem~\ref{ThmConeBound}, for any sequence
of stepsizes in the interval $(0,1)$ satisfying the
bound~\eqref{EqnStepBound}, we have
\begin{align}
  \label{EqnDetClaim}
\idnorm{\thetait{\titer+1} - \thetastar} & \leq \stepit{\titer} \left
\{ \frac{\idnorm{\thetait{1} - \thetastar}}{\stepit{1}} + \contract
\sum_{\iter=1}^\titer \idnorm{\pathit{\iter} } \right \} +
\idnorm{\pathit{\titer+1}},
\end{align}
for all iterations $\titer = 1, 2, \ldots$.
\end{corollary}
\begin{proof}
  Define the error $\DeltaIt{\titer} = \thetait{\titer} - \thetastar$
  at iteration $\titer$.  Using the definitions~\eqref{EqnDefnDetit}
  and~\eqref{EqnDefnAvarit} of $\detit{\titer}$ and $\avarit{\titer}$
  respectively, an inductive argument yields
\begin{subequations}   
  \begin{align}
    \label{EqnDetGenRecur}
\detit{\titer+1} & = \prod_{\iter = 1}^\titer \left(1 - (1- \contract)
\stepit{\iter} \right) \idnorm{\DeltaIt{1}}, \\
\label{EqnAvarGenRecur}
\avarit{\titer + 1} & = \contract \stepit{\titer}
\idnorm{\pathit{\titer}} + \contract \sum_{\iter=1}^{\titer-1} \left
\{ \prod_{j = \iter + 1}^\titer \left(1 - (1-\contract) \stepit{j}
\right) \right \} \stepit{\iter} \idnorm{\pathit{\iter} }.
\end{align}
\end{subequations}
By applying the stepsize bound~\eqref{EqnStepBound} repeatedly to the
recursion~\eqref{EqnDetGenRecur}, we find that $\detit{\titer+1} \leq
\frac{\stepit{\titer}}{\stepit{1}} \idnorm{\DeltaIt{1}}$.  Applying
this same identity to the recursion~\eqref{EqnAvarGenRecur} yields the
bound $\avarit{\titer+1} \leq \contract \stepit{\titer}
\sum_{\iter=1}^\titer \idnorm{\pathit{\iter}}$.  Combining these two
inequalities, along with the additional $\pathit{\titer + 1}$ term
from the bound~\eqref{EqnConeBound} in Theorem~\ref{ThmConeBound},
yields the claim~\eqref{EqnDetClaim}.
\end{proof}

It is worth pointing out why the linear stepsize $\stepit{\titer} =
\titer^{-1}$ is excluded from our theory.  If we adopt this stepsize
choice and substitute into the recursion~\eqref{EqnDetGenRecur}, then
we find that
\begin{align*}
\frac{\detit{\titer+1}}{\idnorm{\thetait{1} - \thetastar}} & =
\prod_{\iter = 1}^\titer \left(1 - \frac{(1- \discount)}{\iter}
\right) \approx \exp \left( - (1-\discount) \sum_{\iter=1}^\titer
\iter^{-1} \right) \approx \left(\frac{1}{\titer}
\right)^{1-\discount}.
\end{align*}
This behavior makes clear that an unrescaled linear stepsize
$\stepit{\titer} = \titer^{-1}$ will lead to bounds with exponential
dependence on $\frac{1}{1-\discount}$.  It should be noted that this
kind of sensitivity to the choices of constants is well-documented
when using linear stepsizes for stochastic optimization; e.g., see
Section 2.1 of Nemirovski et al.~\cite{NemJudLanSha09} for some
examples showing slow rates when the strong convexity constant is
mis-estimated.  As we discuss at more length in
Section~\ref{SecQlearnComparison}, this type of exponential scaling
has also been documented in past work on
$Q$-learning~\cite{Sze97,EveMan03}.

\begin{corollary}[Bounds for polynomial stepsizes]
\label{CorDetClaimPoly}
Under the assumptions of Theorem~\ref{ThmConeBound}, consider the
sequence of stepsizes $1/\titer^{\omega}$ for some $\omega \in (0,1)$.
Then \mbox{for all iterations $\titer = 1, 2, \ldots$,} we have
\begin{align}
\label{EqnDetClaimPoly}
\idnorm{\thetait{\titer+1} - \thetastar} & \leq e^{- \frac{1 -
    \contract}{1-\omega} (\titer^{1 - \omega} - 1)}
\idnorm{\thetait{1} - \thetastar} + e^{- \frac{1-\contract}{1-\omega}
  \titer^{1-\omega}} \sum_{\iter=1}^\titer \frac{e^{
    \frac{1-\contract}{1-\omega} \iter^{1-\omega}}}{\iter^{\omega}}
\idnorm{\pathit{\iter}} + \idnorm{\pathit{\titer+1}}.
\end{align}
\end{corollary}
\begin{proof}
Observe that both of the recursions~\eqref{EqnDetGenRecur}
and~\eqref{EqnAvarGenRecur} hold for general stepsizes in the interval
$(0,1)$.  In order to simplify these expressions, we need to bound the
products of various stepsizes.  We claim that for any positive
integers $T_1 > T_0$, we have
\begin{align}
  \label{EqnStepProdBound}
    \prod_{\iter = T_0}^{T_1} \left( 1 -
    \frac{1-\contract}{\iter^\omega} \right) & \leq \exp \Big( -
    \frac{1-\contract}{1-\omega} (T_1^{1-\omega} - T_0^{1-\omega})
    \Big).
  \end{align}
The proof is straightforward.  From the inequality $\log(1 - s) \leq
-s$, valid for $s \in (0,1)$, we find that $\log \left [\prod_{\iter =
    T_0}^{T_1} \Big( 1 - \frac{1-\contract}{\iter^\omega} \Big)
  \right] \leq - (1-\contract) \sum_{\iter=T_0}^{T_1}
\iter^{-\omega}$.  Now since the function $t \mapsto t^{-\omega}$ is
decreasing on the positive real line, we have
\begin{align*}
  \sum_{\iter=T_0}^{T_1} \iter^{-\omega} & \geq \int_{T_0}^{T_1}
  t^{-\omega} dt \; = \; \frac{1}{1-\omega} \big( T_1^{1-\omega} -
  T_0^{1-\omega} \big).
\end{align*}
Combining the pieces yields the claimed bound.
\end{proof}


\section{Applications to $Q$-learning}
\label{SecQlearn}

We now turn to the consequences of our general results for the problem
of $Q$-learning in the tabular setting.

\subsection{Background and set-up}
\label{SecQlearnBack}
Here we provide only a very brief introduction to Markov decision
processes and the $Q$-learning algorithm; the reader can consult
various standard sources (e.g.,~\cite{Puterman05, SutBar18,
  Bertsekas_dyn1, BerTsi96, Sze09}) for more background.  We consider
a Markov decision process (MDP) with a finite set of possible states
$\StateSpace$, and a finite set of possible actions $\ActionSpace$.
The dynamics are probabilistic in nature and influenced by the
actions: performing action $\action \in \ActionSpace$ while in state
$\state \in \StateSpace$ causes a transition to a new state, randomly
chosen according to a probability distribution denoted
$\Tfun{\state}{\action}{\cdot}$.  Thus, underlying the MDP is a family
of probability transition functions $\{ \Tfun{\state}{\action}{\cdot}
\mid (\state, \action) \in \StateSpace \times \ActionSpace \}$.  The
reward function $\reward$ maps state-action pairs to real numbers, so
that $\reward(\state, \action)$ is the reward received upon executing
action $\action$ while in state $\state$.  A deterministic policy
$\pol$ is a mapping from the state space to the action space, so that
action $\pol(\state)$ is taken when in state $\state$.

For a given policy $\pol$, the \emph{$Q$-function or state-action
  function} measures the expected discounted reward obtained by
starting in a given state-action pair, and then following the policy
$\pol$ in all subsequent iterations.  More precisely, for a given
discount factor $\discount \in (0,1)$, we define
\begin{align}
\Qpol(\state, \action) & = \Exs \left [ \sum^\infty_{\titer=0}
  \discount^\titer \reward(\state_\titer, \action_\titer) \mid
  \state_0 = \state, \action_0 = \action \right] \qquad \mbox{where
  $\action_\titer = \pol(\state_\titer)$ for all $\titer \geq 1$.}
\end{align}
Naturally, we would like to choose the policy $\pol$ so as to optimize
the values of the $Q$-function.  From the classical theory of finite
Markov decision processes~\cite{Puterman05, SutBar18, BerTsi96}, this
task is equivalent to computing the unique fixed point of the Bellman
operator.  The Bellman operator is a mapping from
$\real^{|\StateSpace| \times |\ActionSpace|}$ to itself, whose
$(\state, \action)$-entry is given by
\begin{align}
\label{EqnPopBellman}
  \Bellman(\theta)(\state, \action) & \defn \reward(\state, \action) +
  \discount \Exs_{\xstate{}'} \max_{\action' \in \ActionSpace}
  \theta(\xstate{}', \action') \qquad \mbox{where $\xstate{}' \sim
    \Tfun{\state}{\action}{\cdot}$.}
\end{align}
It is well-known that $\Bellman$ is a $\discount$-contraction with
respect to the $\ell_\infty$-norm, meaning that
\begin{subequations}
\begin{align}
\|\Bellman(\theta) - \Bellman(\theta')\|_\infty & \leq \discount
\|\theta - \theta'\|_\infty \quad \mbox{for all $(\state, \action) \in
  (\StateSpace, \ActionSpace)$,}
\end{align}
where the \emph{$\ell_\infty$ or sup norm} is defined in the usual
way---viz.
\begin{align}
\|\theta\|_\infty \defn \max \limits_{(\state, \action)
  \in \StateSpace \times \ActionSpace} |\theta(\state, \action)|.
\end{align}
\end{subequations}
It is this contractivity that guarantees the existence and uniqueness
of the fixed point $\thetastar$ of the Bellman operator (i.e., for
which $\Bellman(\thetastar) = \thetastar$).

In the context of reinforcement learning, the transition dynamics
$\{\Tfun{\state}{\action}{\cdot}, \; (\state, \action) \in \StateSpace
\times \ActionSpace \}$ are unknown, so that it is not possible to
exactly evaluate the Bellman operator. Instead, given some form of
random access to these transition dynamics, our goal is to compute an
approximation to the optimal $Q$-function on the basis of observed
state-action pairs.  Watkins and Dayan~\cite{WatDay92} introduced the
idea of $Q$-learning, a form of stochastic approximation designed to
compute the optimal $Q$-function.  One can distinguish between the
synchronous and asynchronous forms of $Q$-learning; we focus on the
former here.\footnote{Given bounds on the behavior of synchronous
  $Q$-learning, it is possible to transform them into guarantees for
  the asynchronous model via notions such as the cover time of the
  underlying Markov process; we refer the reader to the
  papers~\cite{EveMan03, Aza11} for instances of such conversions.} In
the synchronous setting of $Q$-learning, we make observations of the
following type.  At each time $\titer = 1, 2, \ldots$ and for each
state-action pair $(\state, \action)$, we observe a sample
$x_\titer(\state, \action)$ drawn according to the transition function
$\Tfun{\state}{\action}{\cdot}$.  Equivalently stated, we observe a
random matrix $X_\titer \in \real^{|\StateSpace| \times
  |\ActionSpace|}$ with independent entries, in which the entry
indexed by $(\state, \action)$ is distributed according to
$\Tfun{\state}{\action}{\cdot}$.

Based on these observations, the synchronous form of $Q$-learning
algorithm generates a sequence of iterates $\{\thetait{\titer}
\}_{\titer \geq 1}$ according to the recursion
\begin{align}
\label{EqnQiteration}
\thetait{\titer + 1} & = ( 1- \stepit{\titer}) \thetait{\titer} +
\stepit{\titer} \EmpBellmanIter{\titer}(\thetait{\titer}).
\end{align}
Here $\EmpBellmanIter{\titer}$ is a mapping from $\real^{|\StateSpace|
  \times |\ActionSpace|}$ to itself, and is known as the
\emph{empirical Bellman operator}: its $(\state, \action)$-entry is
given by
  \begin{align}
  \label{EqnEmpBellman}
  \EmpBellmanIter{\titer}(\theta)(\state, \action) = \reward(\state,
  \action) + \discount \max_{\action' \in \ActionSpace} \theta \big(
  \xstate{\titer}, \action' \big) \qquad \mbox{where $\xstate{\titer}
    \equiv \xstate{\titer}(\state, \action) \sim
    \Tfun{\state}{\action}{\cdot}$.}
\end{align}
By construction, for any fixed $\theta$, we have $\Exs[
  \EmpBellmanIter{\titer}(\theta)] = \Bellman(\theta)$, so that the
empirical Bellman operator~\eqref{EqnEmpBellman} is an unbiased
estimate of the population Bellman operator~\eqref{EqnPopBellman}.

There are different ways in which we can express the $Q$-learning
recursion~\eqref{EqnQiteration} in a form suitable for the application
of Theorem~\ref{ThmConeBound} and its corollaries.  One very natural
approach, as followed in some past work on the problem
(e.g.,~\cite{Tsi94,JaaJorSin94, BerTsi96,EveMan03}), is to rewrite the
$Q$-learning update~\eqref{EqnQiteration} as an application of the
population Bellman operator with noise.  In particular, we can write
\begin{align}
\thetait{\titer + 1} & = ( 1- \stepit{\titer}) \thetait{\titer} +
\stepit{\titer} \big \{ \Bellman(\thetait{\titer}) +
\epsnoiseit{\titer} \big \},
\end{align}
where the noise matrix $\epsnoiseit{\titer} =
\EmpBellmanIter{\titer}(\thetait{\titer}) -
\Bellman(\thetait{\titer})$ is zero-mean, conditioned on
$\thetait{\titer}$.  Theorem~\ref{ThmConeBound} and its corollaries
can then be applied with the orthant cone and the $\ell_\infty$ norm,
along with the operators $\HopIter{\titer} \defn \Bellman$ and
quasi-contraction coefficients $\contiter{\titer} = \discount$ for all
iterations $\titer \geq 1$.

For our purposes, it turns out to be more convenient to apply our
general theory with a \emph{different} and time-varying
choice---namely, with $\HopIter{\titer} \defn \EmpBellmanIter{\titer}$
for each $\titer \geq 1$.  This choice satisfies the required
assumptions, since it can be verified that each one of the random
operators $\EmpBellmanIter{\titer}$ is monotonic with respect to the
orthant ordering, and moreover
\begin{align*}
\| \EmpBellmanIter{\titer}(\theta) -
\EmpBellmanIter{\titer}(\thetastar) \|_\infty & \leq \discount
\|\theta - \thetastar\|_\infty \qquad \mbox{for all $\theta$.}
\end{align*}
Setting $\HopIter{\titer} = \EmpBellmanIter{\titer}$ leads to
effective noise variables (as defined in
equation~\eqref{EqnEffectiveNoise}) of the form
\begin{align}
\label{EqnEasyEffective}
  \wnoiseit{\titer} \defn \EmpBellmanIter{\titer}(\thetastar) -
  \Bellman(\thetastar).
\end{align}
These effective noise variables are especially easy to control.  In
particular, note that $\{\wnoiseit{\titer} \}_{\titer \geq 1}$ is an
i.i.d. sequence of random matrices with zero mean, where entry
$(\state, \action) \in \StateSpace \times \ActionSpace$ has variance
\begin{align}
  \label{EqnDefnQvariance}
  \sigma^2(\thetastar)(\state, \action) & \defn \discount^2
  \Exs_{\widetilde{\xstate{}}} \left[ \Big( \max_{\widetilde{\action}
      \in \ActionSpace} \thetastar(\widetilde{\xstate{}},
    \widetilde{\action}) - \Exs_{\xstate{}'} \max_{\action'
      \in \ActionSpace} \thetastar(\xstate{}', \action') \Big)^2
    \right].
\end{align}
Here the expectations $\Exs_{\widetilde{\xstate{}}}$ and
$\Exs_{\xstate{}'}$ are both computed over
$\Tfun{\state}{\action}{\cdot}$.


\subsection{Non-asymptotic guarantees for $Q$-learning}
\label{SecQlearnResults}

With this set-up, we are now equipped to state some non-asymptotic
guarantees for $Q$-learning.  These bounds involve the quantity $D
\defn |\StateSpace| \times |\ActionSpace|$, corresponding to the total
number of state-action pairs, as well as the \emph{span seminorm} of
$\thetastar$ given by
\begin{align}
  \label{EqnSpanNorm}
\spannorm{\thetastar} = \max_{(\state, \action) \in \StateSpace
  \times \ActionSpace} \thetastar(\state, \action) - \min_{(\state,
  \action) \in \StateSpace \times \ActionSpace} \thetastar(\state,
\action).
\end{align}
Note that this is a seminorm (as opposed to a norm), since we have
$\spannorm{\thetastar} = 0$ whenever $\thetastar$ is constant for all
state-action pairs.  See \S 6.6.1 of Puterman~\cite{Puterman05} for
further background on the span seminorm and its properties.  Finally,
we also define the maximal standard deviation
\begin{align}
  \|\myqvar{\thetastar}\|_\infty = \sqrt{\max_{(\state, \action)
      \in \StateSpace \times \ActionSpace}
    \myqvarsq{\thetastar}(\state, \action)},
\end{align}
where the variance $\myqvarsq{\thetastar}(\state, \action)$ was
previously defined in equation~\eqref{EqnDefnQvariance}.

With these definitions in place, we are now ready to state bounds on the
expected $\ell_\infty$-norm error for $Q$-learning with rescaled linear
stepsizes:
\begin{corollary}[$Q$-learning with rescaled linear stepsize]
  \label{CorQlearnLinear}
Consider the step size choice $\stepit{\titer} = \frac{1}{1 +
  (1-\discount) \titer}$. Then there is a universal constant $\unicon$
such that for all iterations $\titer = 1, 2, \ldots$, we have
\begin{align}
  \label{EqnQlearnLinear}
\Exs[\| \thetait{\titer + 1} - \thetastar\|_\infty] & \leq \frac{
  \|\thetait{1} - \thetastar \|_\infty }{1 + (1 - \discount) \titer} +
\frac{\unicon}{1- \discount} \left \{
\frac{\|\myqvar{\thetastar}\|_\infty \sqrt{\log (2 D)} }{\sqrt{1 +
    (1-\discount) \titer}} + \frac{ \spannorm{\thetastar} \log \left(2 e
  D (1 + (1-\discount) \titer) \right) }{1 + (1-\discount) \titer}
\right \}.
\end{align}
\end{corollary}
\noindent A few remarks about the bound~\eqref{EqnQlearnLinear} are in
order.  Naturally, the first term (involving $\|\thetait{1} -
\thetastar\|_\infty$) measures how quickly the error due to an
initialization $\thetait{1} \neq \thetastar$ decays.  The rate for
this term is $1/\titer$, which is to be expected with a linearly
decaying step size.  The second term in curly braces arises from the
fluctuations of the noise in $Q$-learning, in particular via a
Bernstein bound (see Lemma~\ref{LemExpectedBound}). The term with
$\|\myqvar{\thetastar}\|_\infty$ corresponds to the standard deviation
of the effective noise terms~\eqref{EqnEasyEffective} whereas the term
with $\spannorm{\thetastar}$ arises from the boundedness of the
noises.  Finally, while we have stated a bound on the expected error,
it is also possible to derive a high probability bound: in particular,
if we replace the $\log(D)$ terms with $c \log(D \titer/\delta)$ for a
universal constant $c$, then the bounds hold with probability at least
$1-\delta$.  (See Lemma~\ref{LemNoiseRecursions} in
Appendix~\ref{AppLemNoiseRecursions} for a bound on the moment
generating function of the relevant noise terms.)  \\

\noindent Next we analyze the case of $Q$-learning with a
polynomial-decaying stepsize.
\begin{corollary}[$Q$-learning with polynomial stepsize]
  \label{CorQlearnPoly}
Consider the step size choice $\stepit{\titer} = 1/\titer^{\omega}$
for some $\omega \in (0,1)$.  Then there is a constant $\uniom$,
universal apart from dependence on $\omega$, such that for all
iterations \mbox{$\titer \geq \big(\frac{3 \omega}{ 2(1 -
    \discount)})^{\frac{1}{1-\omega}}$,} we have
\begin{multline}
  \label{EqnQlearnPoly}
\Exs \| \thetait{\titer + 1} - \thetastar\|_\infty \leq e^{-
  \frac{1-\discount}{1-\omega} (\titer^{1-\omega}- 1)} \left \{
\|\thetait{1} - \thetastar \|_\infty + \uniom
(1-\discount)^{-\frac{1}{1-\omega}} \right \} \\
+ \frac{\uniom}{1 - \discount} \left \{ \frac{\|
  \myqvar{\thetastar}\|_\infty \sqrt{\log (2 D)} }{\titer^{\omega/2}}
+ \frac{\spannorm{\thetastar} \log (2 D) }{\titer^{\omega}} \right \}.
\end{multline}
\end{corollary}
At a high level, the interpretation of this bound is similar to that
of the bound in Corollary~\ref{CorQlearnLinear}: the first term
corresponds to the initialization error, whereas the second term
corresponds to the fluctuations induced by the stochasticity of the
update.  When taking the much larger polynomial stepsizes---in
contrast to the linear stepsize case---the initialization error
vanishes much more quickly, in particular as an exponential function
of $(1-\discount) \titer^{1-\omega}$.  On the other hand, the noise
terms exhibit larger fluctuations---with the two terms in the
Bernstein bound scaling as $\titer^{-\omega/2}$ and $\titer^{-\omega}$.


\subsection{Comparison to past work}
\label{SecQlearnComparison}

There is a very large body of work on $Q$-learning in different
settings, and studying its behavior under various criteria.  Here we
focus only on the subset of work that has given bounds on the
$\ell_\infty$-error for discounted problems, which is most relevant
for direct comparison to our results.  The $Q$-learning algorithm was
initially introduced and studied by Watkins and Dayan~\cite{WatDay92}.
General asymptotic results on the convergence of $Q$-learning were
given by Tsitsiklis~\cite{Tsi94} and Jaakkola et
al.~\cite{JaaJorSin94}, who made explicit connection to stochastic
approximation.  Szepesv\'{a}ri~\cite{Sze97} gave an asymptotic
analysis showing (among other results) that the convergence rate of
$Q$-learning with linear stepsizes can be exponentially slow as a
function of $1/(1-\discount)$.  Bertsekas and
Tsitsiklis~\cite{BerTsi96} provided a general framework for the
analysis of stochastic approximation of the $Q$-learning type, and
used it to provide asymptotic convergence guarantees for a broad range
of stepsizes.  Using this same framework, Even-Dar and
Mansour~\cite{EveMan03} performed an epoch-based analysis that led to
non-asymptotic bounds on the behavior of $Q$-learning, both for the
\emph{non-rescaled} linear stepsize $\stepit{\titer} = \titer^{-1}$,
and the polynomial stepsizes $\stepit{\titer} = \titer^{-\omega}$ for
$\omega \in (0,1)$.  It is these non-asymptotic results that are most
directly comparable to our Corollary~\ref{CorQlearnPoly}.

In order to make some precise comparisons, consider the class of
MDPs in which the reward function is uniformly as bounded
\begin{align}
  \label{EqnRmax}
  \max_{(\state, \action) \in \StateSpace \times \ActionSpace}
  |\reward(\state, \action)| \leq \rmax.
\end{align}
Bounds from past work~\cite{EveMan03} are given in terms of the
\emph{iteration complexity} $T_\omega(\epsilon, \discount, \rmax)$ of
the algorithms, meaning the minimum number of iterations $T$ required
to drive the expected $\ell_\infty$-error\footnote{In fact, they
  stated their results as high-probability bounds but up to some
  additional logarithmic factors, these are the same as the bounds on
  expected error.}  $\Exs \|\thetait{T} - \thetastar\|_\infty$ below
$\epsilon$.


\subsubsection{Linear stepsizes}

For the unrescaled linear stepsizes $\stepit{\titer} = \titer^{-1}$,
Even-Dar and Mansour~\cite{EveMan03} proved a pessimistic result:
namely, that $Q$-learning with this step size has an iteration
complexity that grows exponentially in the quantity $1/(1-\discount)$.
As noted previously, earlier work by Szepesv\'{a}ri~\cite{Sze97} had
given an asymptotic analogue of this poor behavior of $Q$-learning
with this linear stepsize.  Moreover, as we discussed following
Corollary~\ref{CorDetClaim}, this type of exponential scaling will
also arise if our general machinery is applied with the ordinary
linear stepsize.

Let us now turn to the $\ell_\infty$ bounds given by
Corollary~\ref{CorQlearnLinear} using the \emph{rescaled} linear
stepsize \mbox{$\stepit{\titer} = \frac{1}{1 + (1-\discount)
    \titer}$.} Translating these bounds into iteration complexity, we
find that taking
\begin{align}
  \label{EqnMyLinear}
\TLR(\epsilon, \discount, \thetastar) & \precsim
\left(\frac{\|\thetait{1} - \thetastar\|_\infty}{1-\discount} +
\frac{\spannorm{\thetastar}}{(1-\discount)^2 } \right) \left (
\frac{1}{\epsilon} \right) + \left(
\frac{\|\myqvar{\thetastar}\|_\infty^2}{(1-\discount)^3 \epsilon^2}
\right)
\end{align}
iterations is sufficient to guarantee $\epsilon$-accuracy in expected
$\ell_\infty$ norm.  Here the notation $\precsim$ denotes an
inequality that holds with constants and log factors dropped, so as to
simplify comparison of results.  As we will see momentarily, all the
quantities in our bound scale at most polynomially as a function of
$1/(1-\discount)$, demonstrating the importance of using the rescaled
linear stepsize.


\subsubsection{Polynomial stepsizes}

We now turn to the polynomial step sizes $\stepit{\titer} =
\titer^\omega$ for $\omega \in (0,1)$, and compare our results to past
work. For these polynomial step sizes, Even-Dar and
Mansour~\cite{EveMan03} (in Theorem 2 of their paper) proved that for
any MDP with $\rmax$-bounded rewards and discount factor $\discount$,
it suffices to take at most
\begin{align}
  \label{EqnEveMan}
  T_\omega(\epsilon, \discount, \rmax) & \precsim \left(
  \frac{\rmax^2}{(1-\discount)^4 \epsilon^2}
  \right)^{\frac{1}{\omega}} + \left \{\frac{1}{1-\discount} \log
  \left(\frac{\rmax }{(1-\discount) \, \epsilon} \right) \right
  \}^{\frac{1}{1-\omega}}
\end{align}
in order to drive the error below $\epsilon$.  Here as before, our
notation $\precsim$ indicates that we are dropping constants and other
logarithmic factors (including those involving $\log D$).

On the other hand, Corollary~\ref{CorQlearnPoly} in this paper
guarantees that for a $\discount$-discounted MDP with optimal
$Q$-function $\thetastar$, initializing at $\thetait{1} = 0$ and
taking
\begin{align}
\label{EqnMyPoly}
  T_\omega(\epsilon, \discount, \thetastar) & \precsim \left(
  \frac{\|\myqvar{\thetastar}\|^2_\infty}{(1-\discount)^2 \epsilon^2}
  \right)^{\frac{1}{\omega}} + \left(
  \frac{\spannorm{\thetastar}^2}{(1-\discount)^2 \epsilon^2}
  \right)^{\frac{1}{2 \omega}} + \left \{\frac{1}{1-\discount} \log
  \left( \frac{\rmax}{(1-\discount) \, \epsilon} \right) \right
  \}^{\frac{1}{1-\omega}}
\end{align}
steps is sufficient to achieve an $\epsilon$-accurate estimate. 
As we will see momentarily, our guarantee~\eqref{EqnMyPoly} reduces to the
earlier guarantee~\eqref{EqnEveMan} in the worst-case setting.


\subsubsection{Worst-case guarantees}

In our bounds for $Q$-learning, the $\thetastar$-specific difficulty
enters via the span seminorm $\spannorm{\thetastar}$ and the maximal
standard deviation $\|\myqvar{\thetastar}\|_\infty$.  In this section,
we bound these quantities in a worst-case sense, and recover some
guarantees known from past work. In particular, let $\MDP(\discount,
\rmax)$ denote the set of all optimal $Q$-functions that can be
obtained from a $\discount$-discounted MDP with an $\rmax$-uniformly
bounded reward function (as in equation~\eqref{EqnRmax}).

\begin{lemma}
\label{LemSTDBound}  
Over the class $\MDP(\discount, \rmax)$, we have the uniform bounds
\begin{subequations}
\begin{align}
  \label{EqnQbounds}
  \sup_{\thetastar \in \MDP(\discount, \rmax)} \spannorm{\thetastar}
\leq
2  \sup_{\thetastar \in \MDP(\discount, \rmax)} \|\thetastar\|_\infty
& \leq \frac{2 \discount \rmax}{1-\discount} \quad \mbox{and} \\
\sup_{\thetastar \in \MDP(\discount, \rmax)}
\|\sigma(\thetastar)\|_\infty & \leq \frac{\rmax}{1-\discount}
\end{align}
\end{subequations}
\end{lemma}

Lemma~\ref{LemSTDBound} allows us to derive uniform versions of our
previous iteration complexity bounds.  For simplicity, we assume
initialization at $\thetait{1} = 0$, so that $\|\thetait{1} -
\thetastar\|_\infty = \|\thetastar\|_\infty$.  For the rescaled linear
stepsize, for all $\epsilon \in (0, \rmax)$, we have
\begin{align}
  \label{EqnMyLinearWorst}
  \sup_{\thetastar \in \MDP(\discount, \rmax)} \TLR(\epsilon,
  \discount, \thetastar) & \precsim \left(
  \frac{\rmax^2}{(1-\discount)^5 \epsilon^2} \right).
\end{align}
On the other hand, for the polynomial step size, we have
\begin{align}
\label{EqnMyPolyWorst}
\sup_{\thetastar \in \MDP(\discount, \rmax)} T_\omega(\epsilon,
\discount, \thetastar) & \precsim \left(
\frac{\rmax^2}{(1-\discount)^4 \epsilon^2} \right)^{\frac{1}{\omega}}
+ \left \{\frac{1}{1-\discount} \log \left(\frac{\rmax}{(1 -
  \discount) \, \epsilon} \right) \right \}^{\frac{1}{1-\omega}}.
\end{align}
Note that this bound shows a trade-off between the two terms as a
function of $\omega \in (0,1)$.  Setting $\omega = \frac{4}{5}$
optimizes the trade-off in terms of $1/(1-\discount)$, and yields a
bound of the order $\left(\frac{\rmax}{\epsilon} \right)^\frac{5}{2}
\frac{1}{(1 - \discount)^5}$, equivalent to that proved in past
work~\eqref{EqnEveMan}.  Note that this bound has the same scaling in
$\discount$ as the linear rescaled bound~\eqref{EqnMyLinearWorst}, but
with worse behavior in the ratio $\rmax/\epsilon$.

\subsection{Simulation study of $\discount$-dependence}
\label{SecHard}

It is natural to wonder whether or not the bounds given in
Corollaries~\ref{CorQlearnLinear} and~\ref{CorQlearnPoly} give sharp
scalings for the dependence of $Q$-learning on the variance
$\sigma(\thetastar)$ and span seminorm $\spnorm{\thetastar}$.  In this
section, we provide empirical evidence for the sharpness of our
bounds, as well as a case in which they fail to be sharp.

\begin{figure}[h]
\begin{center}
  \begin{tabular}{ccc}
    \widgraph{0.3\textwidth}{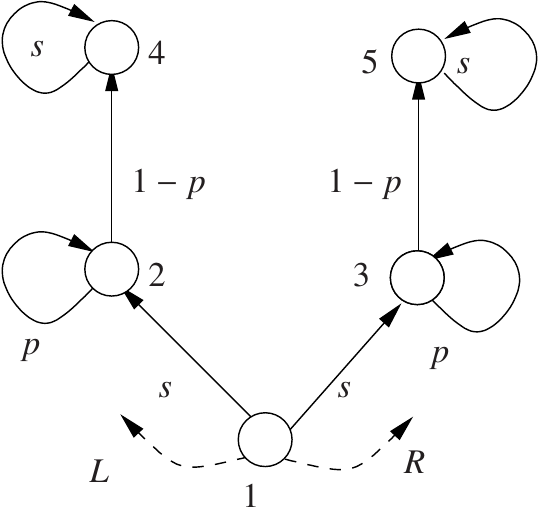} & &
    \widgraph{0.5\textwidth}{\figdir/fig_logplots_linear_vs_poly} \\
    (a) & & (b)
  \end{tabular}
  \caption{(a) Transition diagram of a class of MDPs for $Q$-learning,
    inspired by an example from Azar et al.~\cite{AzaMunKap13}.  This
    simplification of their example consists of state space
    $\StateSpace = \{1, 2, 3, 4, 5 \}$ and action space $\ActionSpace
    = \{L, R \}$.  When in state $\state = 1$, taking action $\action
    = L$ leads to state $2$ with probability $s = 1$, while taking
    action $\action = R$ leads to state $4$ with probability $s = 1$.
    When in state $2$, regardless of the action, we stay in state $2$
    with probability $p$ and transition to state $4$ with probability
    $1-p$.  State $3$ behaves in an analogous manner, transitioning to
    state $5$ with probability $1-p$ and remaining fixed with
    probability $p$.  States $4$ and $5$ are absorbing states.  States
    $2$ and $3$ have reward $1$ for either action; all remaining
    state-action pairs have zero reward. We are free to choose the
    parameter $p \in (0,1)$ in order to vary the difficulty of the
    problem.  (b) Log-log plots of the $\ell_\infty$-norm error of the
    $Q$-learning iterates versus iteration number.  Standard theory
    predicts that on the log-log scale, the linear stepsize should
    yield a line with slope $-1/2$, whereas the polynomial stepsize
    should yield a line with slope $-\omega/2$.  Of more interest to
    us is the rate at which these curves shift upwards as $\discount$
    increases towards $1$.}
  \label{FigHard}
\end{center}
\end{figure}

In order to do so, we consider a class of MDPs introduced in past work
by Azar et al.~\cite{AzaMunKap13}, and used to prove minimax lower
bounds.  For our purposes---namely, exploring sharpness with the
discount $\discount$---it suffices to consider an especially simple
instance of these ``hard'' problems. As illustrated in
Figure~\ref{FigHard}(a), this MDP consists of a five element space
$\StateSpace = \{1, 2, \ldots, 5 \}$, and a two element action space
$\ActionSpace = \{L, R \}$, shorthand for ``left'' and ``right''
respectively.  When in state $1$, taking action $L$ yields a
deterministic transition (i.e., with probability $s = 1$) to state $2$
whereas taking action $R$ leads to a deterministic transition to state
$3$.  When in state $2$, taking either action leads to a transition to
state $4$ with probability $1-p$, and remaining in state $2$ with
probability $p$.  (The same assertion applies to the behavior in state
$3$, with state $4$ replaced by state $5$.)  Finally, both states $4$
an $5$ are absorbing states.  The reward function in zero in every
state except for states $2$ and $3$, for which we have
\begin{align*}
  \reward(2, L) = \reward(2, R) = \reward(3, L) = \reward(3, R) = 1.
\end{align*}
A straightforward computation yields that the optimal $Q$-function has
the form
\begin{align*}
  \thetastar(\state, \action) & = \begin{cases}
    \frac{\discount}{1 - p \discount} & \mbox{for $\state = 1$} \\
    \frac{1}{1 - p \discount} & \mbox{for $\state \in \{2, 3 \}$} \\
    0 & \mbox{for $\state \in \{4, 5 \}$}
  \end{cases}
\end{align*}
For any $\discount \in (1/4, 1)$, it is valid to set $p = \frac{4
  \discount -1}{3 \discount}$.

If we run $Q$-learning either with a rescaled linear stepsize (as in
Corollary~\ref{CorQlearnLinear}) or a polynomial stepsize (as in
Corollary~\ref{CorQlearnPoly}), then as shown in
Figure~\ref{FigHard}(b), we see convergence at the rate
$\titer^{-1/2}$ for the linear stepsize, and $\titer^{-\omega/2}$ for
the polynomial stepsize.  This is consistent with the theory, and
standard for stochastic approximation.  Of most interest to us is the
behavior of the curves as the discount factor $\discount$ is changed;
as seen in Figure~\ref{FigHard}(b), the curves shift upwards,
reflecting the fact that problems with larger value of $\discount$ are
harder.  We would like to understand these shifts in a quantitative manner.


\subsubsection{Behavior of $\spannorm{\thetastar}$ and $\|\myqvar{\thetastar}\|_\infty$}

For this particular class of problems, let us compute the quantities
$\spannorm{\thetastar}$ and $\|\myqvar{\thetastar}\|_\infty$ that play
a key role in our bounds.  First, observe that with our choice of $p$
from above, we have $\frac{1}{1 - p \discount} = \frac{3}{3 - (4
  \discount - 1)} \; = \; \frac{3}{4} \left( \frac{1}{1 - \discount}
\right)$, which implies that
\begin{subequations}
\begin{align}
  \spannorm{\thetastar} = \frac{3}{4} \frac{1}{1-\discount} - 0 \; = \;
  \frac{3}{4} \frac{1}{1-\discount}.
\end{align}
Recalling that $\rmax = 1$ in our construction, observe that (up to a
constant factor) this $Q$-function saturates the worst-case upper
bound on $\spannorm{\thetastar}$ from Lemma~\ref{LemSTDBound}.  As for
the maximal standard deviation term, as shown in
Appendix~\ref{AppHard}, as long as $\discount \geq 1/2$, it is sandwiched as
\begin{align}
  \label{EqnHardLower}
 \frac{1}{4 \sqrt{3}} \frac{1}{\sqrt{1- \discount}} \; \leq \;
 \|\myqvar{\thetastar}\|_\infty \; \leq \;
 \frac{1}{\sqrt{1-\discount}}.
\end{align}
\end{subequations}
\begin{figure}[h]
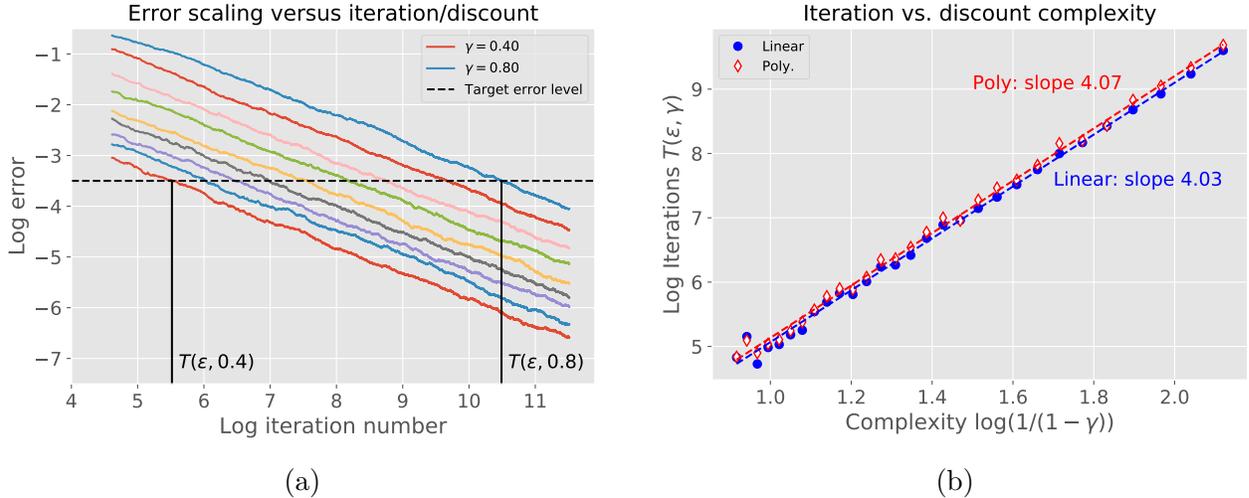

\begin{center}
  \begin{tabular}{cc}
    \widgraph{0.5\textwidth}{\figdir/fig_errorscale_setup} &
    \widgraph{0.5\textwidth}{\figdir/fig_iteration_complexity_last} \\
    (a) &  (b)
  \end{tabular}
  \caption{(a) Illustration of the computation of the iteration
    complexity $T(\epsilon; \discount)$ based on simulation data.  It
    is corresponds to the first time that the average-$\ell_\infty$
    norm error drops below $\epsilon$; varying the discount factor
    $\discount$ corresponds to problems of differing hardness. (b)
    Plots of the computed iteration complexity $T(\epsilon,
    \discount)$ for $\epsilon = e^{-2}$ versus the complexity
    parameter $1/(1- \discount)$ for both the rescaled linear
    stepsize, and the polynomial stepsize with $\omega = 0.75$.  In
    both cases, our theory predicts that these fits should be linear
    on a log-log scale with a slope of $4$.}
  \label{FigErrorScale}
\end{center}
\end{figure}
Consequently, by examining our iteration complexity
bounds~\eqref{EqnMyLinear} and~\eqref{EqnMyPoly}, we expect that for
any fixed $\epsilon \in (0,1)$, the iteration complexity $T(\epsilon,
\discount)$ of $Q$-learning as a function of $\discount$ should be
upper bounded as $(1 - \discount)^{-4}$, and moreover, this bound
should hold for either the rescaled linear stepsize, or the polynomial
stepsize with $\omega = \frac{3}{4}$.  If our bounds cannot be
improved in general, then we expect to see that this predicted bound
is met with equality in simulation.  Accordingly, our numerical
simulations were addressed to testing the correctness of this
prediction.

Figure~\ref{FigErrorScale} illustrates the results of our simulations.
Panel (a) illustrates how the iteration complexity $T(\epsilon,
\discount)$ was estimated in simulation.  For a given algorithm and
setting of $\discount$, we ran the algorithm for $10^6$ steps,
thereby obtaining a path of $\ell_\infty$-norm errors at each
iteration $\titer = 1, \ldots, 10^6$.  We averaged these paths over a
total of $10^3$ independent trials.  Given these Monte Carlo estimates
of the average $\ell_\infty$-error, for a given $\epsilon$, we compute
$T(\epsilon, \discount)$ by finding the smallest iteration $\titer$ at
which the estimated $\ell_\infty$-error falls below $\epsilon$.  Panel
(a) illustrates two instances of this calculation, for the settings
$\discount = 0.4$ and $\discount = 0.8$ respectively.

For the fixed tolerance $\epsilon = e^{-2}$, we repeated this Monte
Carlo estimation procedure in order to estimate the quantity
$T(\epsilon, \discount)$ for each $\discount \in \{0.60, 0.61, \ldots,
0.90 \}$, and then plotted the results on a log-log scale, as shown in
panel (b).  We fit each set of points with ordinary linear regression,
yielding estimates of the slope $\betapoly = 4.07 \pm 0.04$ and
$\betalin = 4.03 \pm 0.05$.  We also applied a $t$-test with the null
hypothesis being that the slope is $4$; it returned $p$-values of
$0.12$ and $0.50$ in the polynomial and linear cases respectively.
Thus, these simulations provide empirical evidence for the sharpness
of our bounds in this ensemble of problems.

\subsubsection{A non-sharp example}

The example from the previous section shows that our bounds are
unimprovable in general.  However, there are instances for which our
bounds are \emph{not} sharp, with one the simplest being an MDP in
which there is a root state that transitions (w.p.  $1/2$) to one of
two other states, both of which are absorbing.  If we set the rewards
at the two absorbing states to be $-1$ and $1$ respectively, then the
associated $Q$-functions at the two absorbing states will be
$-\frac{1}{1-\discount}$ and $\frac{1}{1-\discount}$ respectively.
These highly discrepant values mean that the variance of the empirical
Bellman operator at the root state will scale as
$\frac{1}{(1-\discount)^2}$.  Consequently, our bounds yield
guarantees with the rather pessimistic scaling
$\frac{1}{(1-\discount)^5}$, a behavior that is not observed in
practice.  The issue here is for problems of this type, the quantity
$\sigma(\thetastar)$ is overly pessimistic, as it measures only the
local variance (and not how fluctuations can be diminished by repeated
applications of the Bellman operator).  Resolving this gap requires a
more sophisticated analysis, as in the work of Azar et
al.~\cite{AzaMunKap13} on model-based $Q$-fitting.


\section{Discussion}
\label{SecDiscussion}

The main contribution of this paper (Theorem~\ref{ThmConeBound}) is a
``sandwich'' result for a class of stochastic approximation (SA)
algorithms based on operators that satisfy monotonicity and
quasi-contractivity conditions with respect to an underlying cone.
This general result can be used to derive non-asymptotic error bounds
for SA procedures with various stepsizes.  We then illustrated this
general result by applying it to derive non-asymptotic bounds for
synchronous $Q$-learning applied to discrete state-action problems.
We hope that this general result proves useful in analyses of other
stochastic approximation algorithms.

This paper leaves open various questions.  Our analysis covers both
the cases of linearly decaying stepsizes, as well as the (more slowly
decaying) polynomial choices for some $\omega \in (0,1)$.  The
advantage of polynomial stepsizes is that they are robust to the
choices of constants; in contrast, the linear stepsize behaves badly
unless it is suitably rescaled, requiring knowledge of the contraction
coefficient. On the flipside, the polynomial choices perform more
poorly in damping noise, so that the convergence guarantees show
inferior scaling in the inverse tolerance $(1/\epsilon)$.  A standard
resolution to this difficulty is to perform Polyak-Ruppert
averaging~\cite{NemYu83,PolJud92,Rup88}: that is, to run the algorithm
with the slower polynomial stepsizes, and then average the iterates
along the path with a $1/\titer$ stepsize.  It would be interesting to
extend our non-asymptotic analysis to the the class of SA procedures
considered here combined with such averaging methods.

In the specific context of $Q$-learning, as we noted earlier, our
results also show that standard $Q$-learning is a sub-optimal
algorithm, in that we exhibited an example where its iteration
complexity scales as $(1-\discount)^{-4}$, as opposed to the
$(1-\discount)^{-3}$ scaling that can be achieved by model-based batch
$Q$-iteration~\cite{AzaMunKap13}.  Thus, our work highlights a gap
between standard $Q$-learning as a model-free method and a model-based
method, as has been done in recent work on the LQR
problem~\cite{TuRec19}.  It is natural to wonder whether this gap is
specific to $Q$-learning, or also applies to other algorithms with
essentially equivalent storage and/or computational requirements.  We
note that the speedy $Q$-learning algorithm~\cite{Aza11}, requiring
only twice the storage of $Q$-learning, has been shown, in the setting
considered here, to have worst-case iteration complexity scaling as
$\frac{1}{(1-\discount)^4} \frac{1}{\epsilon^2}$.

Moving beyond the worst-case viewpoint, our bounds on $Q$-learning are
instance-specific, depending on the unknown optimal $Q$-function via
its variance~\eqref{EqnDefnQvariance} and span
seminorm~\eqref{EqnSpanNorm}.  These two quantities vary substantially
over the space of MDPs with $\rmax$-bounded rewards, being much
smaller than their worst-case values for many problems.  Thus, it
would be interesting to develop instance-based lower bounds on the
performance of $Q$-learning and other algorithms.  We note that such
types of localized analysis have been given both in the bandit
literature~\cite{BubCes12}, and in non-parametric
statistics~\cite{WeiFanWai18}.


\subsection*{Acknowledgements}
Thanks to A. L. Z. Pananjady and K. B. B. Khamaru for helpful
discussions.  This work was partially supported by Office of Naval
Research Grant ONR-N00014-18-1-2640 and National Science Foundation
Grant NSF-DMS-1612948.


\appendix

\section{Proof of Theorem~\ref{ThmConeBound}}
\label{AppThmConeBound}

We begin by defining a recentered version of the
updates~\eqref{EqnRecursion}.  Subtracting $\thetastar$ from both
sides yields the recentered recursion
\begin{align*}
\thetait{\titer + 1} - \thetastar & = (1 - \stepit{\titer}) \big(
\thetait{\titer} - \thetastar \big) + \stepit{\titer} \big \{
\HopIter{\titer}(\thetait{\titer}) - \thetastar + \epsnoiseit{\titer}
\big \}.
\end{align*}
Define the operator $\TopIter{\titer}(\Delta) =
\HopIter{\titer}(\thetastar + \Delta) - \HopIter{\titer}(\thetastar)$.
With this notation, we see that the error sequence $\DelIt{\titer}
\defn \thetait{\titer} - \thetastar$ evolves according to the
recursion
\begin{align}
\DelIt{\titer + 1} & = (1 - \stepit{\titer}) \DelIt{\titer} +
\stepit{\titer} \big \{ \TopIter{\titer}(\DelIt{\titer}) +
\underbrace{\HopIter{\titer}(\thetastar) - \thetastar +
  \epsnoiseit{\titer}}_{\wnoiseit{\titer}} \big \},
\end{align}
where we have re-introduced the effective noise variables
$\wnoiseit{\titer}$, as previously defined in
equation~\eqref{EqnEffectiveNoise}.  By our assumptions on
$\HopIter{\titer}$, we see that each one of the recentered operators
$\TopIter{\titer}$ is monotonic with respect to the cone $\Cone$, and
contractive in the sense that $\idnorm{\TopIter{\titer}(\Delta)} \leq
\contiter{\titer} \idnorm{\Delta}$ for all $\Delta \in \LinSpace$.

We first prove the upper bound in equation~\eqref{EqnConeBound} via
induction on the iteration number $\titer$.  Beginning with the base
case, for iteration $\titer = 1$, we have
\begin{align*}
  \DeltaIt{1} & \preceq \idnorm{\thetait{1} - \thetastar}\; \idvec \;
  = \; \detit{1} \idvec + \avarit{1} \idvec + \pathit{1},
\end{align*}
where we have used the initialization conditions $\avarit{1} = 0$ and
$\pathit{1} = \zerovec$.

We now assume that the claim holds at iteration $\titer-1$, and show
that it holds for iteration $\titer$.  We have
\begin{align*}
\DeltaIt{\titer} & = (1 - \stepit{\titer-1}) \DeltaIt{\titer-1} +
\stepit{\titer-1} \TopIter{\titer-1}(\DeltaIt{\titer-1}) +
\stepit{\titer-1} \wnoiseit{\titer-1} \\
& \preceq (1 - \stepit{\titer-1}) \left( \detit{\titer-1}\idvec +
\avarit{\titer-1} \idvec + \pathit{\titer-1} \right) +
\stepit{\titer-1} \TopIter{\titer - 1} \left( \detit{\titer-1}\idvec +
\avarit{\titer-1} \idvec + \pathit{\titer-1} \right) +
\stepit{\titer-1} \wnoiseit{\titer-1},
\end{align*}
where the inequality follows from the inductive assumption; the
monotonicity of the operator $\TopIter{\titer-1}$ with respect to the
cone; and the fact that $\idnorm{e} = 1$.  Now by the
$\contiter{\titer-1}$-contractivity of $\TopIter{\titer-1}$, we have
\begin{multline*}
  \DeltaIt{\titer} \preceq (1 - \stepit{\titer-1}) \left(
  \detit{\titer-1}\idvec + \avarit{\titer-1} \idvec +
  \pathit{\titer-1} \right) + \contiter{\titer-1} \stepit{\titer-1}
  \idnorm{ \detit{\titer-1}\idvec + \avarit{\titer-1} \idvec +
    \pathit{\titer-1} } \; \idvec + \stepit{\titer-1}
  \wnoiseit{\titer-1} \\
 \preceq \underbrace{ \left( 1 - (1-\contiter{\titer-1})
   \stepit{\titer-1} \right) \detit{\titer-1}}_{\detit{\titer}} \idvec
 + \underbrace{\left \{ \left( 1 - (1-\contiter{\titer-1})
   \stepit{\titer-1} \right) \avarit{\titer-1} + \discount
   \stepit{\titer-1} \idnorm{\pathit{\titer-1}} \right
   \}}_{\avarit{\titer}} \idvec \\
  + \underbrace{(1 - \stepit{\titer-1}) \pathit{\titer-1} +
    \stepit{\titer-1} \wnoiseit{\titer-1}}_{\pathit{\titer}},
\end{multline*}
which establishes the claim in part (a).

Turning to the lower bound in part (a), again we proceed via induction
on the iteration number $\titer$.  Beginning with the base case, for
iteration $\titer = 1$, we have
\begin{align*}
  \DeltaIt{1} & \succeq -\idnorm{\DeltaIt{1}} \idvec \; = \;
  -\detit{1} \idvec - \avarit{1} \idvec + \pathit{1},
\end{align*}
as required.

We now assume that the claim holds at iteration $\titer-1$, and show
that it holds for iteration $\titer$.  We have
\begin{align*}
\DeltaIt{\titer} & = (1 - \stepit{\titer-1}) \DeltaIt{\titer-1} +
\stepit{\titer-1} \TopIter{\titer-1}(\DeltaIt{\titer-1}) +
\stepit{\titer-1} \wnoiseit{\titer-1} \\
& \succeq (1 - \stepit{\titer-1}) \left( - \detit{\titer-1}\idvec -
\avarit{\titer-1} \idvec + \pathit{\titer-1} \right) +
\stepit{\titer-1} \TopIter{\titer-1} \left( -\detit{\titer-1}\idvec -
\avarit{\titer-1} \idvec + \pathit{\titer-1} \right) +
\stepit{\titer-1} \wnoiseit{\titer-1},
\end{align*}
where the inequality follows from the inductive assumption, and the
monotonicity of the operator $\TopIter{\titer-1}$ on the cone.  Now by
the $\contiter{\titer-1}$-contractivity of $\TopIter{\titer-1}$, we
have
\begin{multline*}
  \DeltaIt{\titer} \succeq (1 - \stepit{\titer-1}) \left(
  -\detit{\titer-1} \idvec - \avarit{\titer-1} \idvec +
  \pathit{\titer-1} \right) - \contiter{\titer-1} \stepit{\titer-1} \idnorm{
    -\detit{\titer-1} \idvec - \avarit{\titer-1} \idvec +
    \pathit{\titer-1} } \; \idvec + \stepit{\titer-1}
  \wnoiseit{\titer-1} \\
 \succeq \underbrace{ - \left( 1 - (1-\contiter{\titer-1}) \stepit{\titer-1}
   \right) \detit{\titer-1}}_{-\detit{\titer}} \idvec + \underbrace{-
   \left \{ \left( 1 - (1-\contiter{\titer-1}) \stepit{\titer-1} \right)
   \avarit{\titer-1} + \contiter{\titer-1} \stepit{\titer-1}
   \idnorm{\pathit{\titer-1}} \right \}}_{-\avarit{\titer}} \idvec \\
 + \underbrace{(1 - \stepit{\titer-1}) \pathit{\titer-1} +
   \stepit{\titer-1} \wnoiseit{\titer-1}}_{\pathit{\titer}},
\end{multline*}
which completes the proof of the lower bound.


\section{Proofs for $Q$-learning}
\label{SecProofQ}

In this appendix, we collect the proofs of our results on
$Q$-learning, beginning with the statements and proofs of some
auxiliary lemmas in Section~\ref{SecAuxLem} and followed by the proofs
of Corollaries~\ref{CorQlearnLinear} and~\ref{CorQlearnPoly} in
Sections~\ref{AppProofCorQlearnLinear} and~\ref{AppProofCorQlearnPoly}
respectively.


\subsection{Proofs of some auxiliary lemmas}
\label{SecAuxLem}

In this appendix, we collect the statements and/or proofs of various
auxiliary lemmas needed for the proofs of
Corollaries~\ref{CorQlearnLinear} and~\ref{CorQlearnPoly}.

\subsubsection{Proof of Lemma~\ref{LemSTDBound}}
\label{AppLemSTDBound}

The inequality $\spannorm{\thetastar} \leq 2 \|\thetastar\|_\infty$ is
a standard fact~\cite{Puterman05}; it follows by applying the triangle
inequality to the definition of the span seminorm.  Since $\thetastar$
is a fixed point of the Bellman equation~\eqref{EqnPopBellman}, we
have
  \begin{align*}
    \thetastar(\state, \action) & = \reward(\state, \action) +
    \discount \Exs_{\xstate{}'} \max_{\action' \in \ActionSpace}
    \thetastar(\xstate{}', \action') \; \leq \; \rmax + \discount
    \|\thetastar\|_\infty
  \end{align*}
  from which it follows that $\sup \limits_{\thetastar \in
    \MDP(\discount, \rmax)} \|\thetastar\|_\infty \leq
    \frac{\rmax}{1-\discount}$ as claimed.

Turning to the variance bound, from the definition of the empirical
Bellman operator, we have
\begin{align*}
\var\left(\EmpBellman(\thetastar) \right) & \leq 4 \discount^2
\|\thetastar\|_\infty^2 \; \leq \; \frac{4 \discount^2
  \rmax^2}{(1-\discount)^2},
    \end{align*}
as claimed.
    
\subsubsection{Controlling the MGF of non-stationary autoregressive processes}
\label{AppLemNoiseRecursions}

Let $\{\nepsnoiseit{\titer} \}_{\titer \geq 1}$ be a sequence of
random variables and let $\{\stepit{\titer} \}_{\titer \geq 1}$ be a
sequence of stepsizes in $(0,1)$. Define a new sequence of random
variables $\{ \Vio_\titer \}_{\titer \geq 1 }$ via the non-stationary
autoregression
\begin{subequations}
\begin{align}
  \label{EqnAutoregression}
\Vio_{\titer + 1} = (1-\stepit{\titer}) \Vio_\titer + \stepit{\titer}
\nepsnoiseit{\titer},
\end{align}
with initialization $\Vio_1 = 0$.  Suppose that the stepsizes satisfy the
inequality
\begin{align}
  \label{EqnStepInequal}
(1 - \stepit{\titer}) \stepit{\titer-1} & \leq \stepit{\titer}.
\end{align}
\end{subequations}
The two choices of particular interest to us, both of which satisfy
this inequality, are the rescaled linear stepsize $\stepit{\titer}
\defn \frac{1}{1 + (1 - \discount) \titer}$, for which we have
\begin{align*}
(1 - \stepit{\titer}) \stepit{\titer-1} & = \frac{1}{1 + (1-\discount)
    \titer} \; \; \underbrace{\frac{(1-\discount) \titer }{1 +
      (1-\discount) (\titer-1)}}_{ \leq 1} \; \leq \; \stepit{\titer},
\end{align*}
and the polynomial stepsize $\stepit{\titer} = 1/\titer^\omega$, for
which we have
\begin{align*}
(1 - \stepit{\titer}) \stepit{\titer-1} & = \frac{1}{\titer^{\omega}}
  \; \; \underbrace{\frac{\titer^\omega -1 }{(\titer -
      1)^\omega}}_{\leq 1} \; \leq \; \stepit{\titer},
\end{align*}
where the inequality follows from the fact that $(\titer -1)^{\omega}
\geq \titer^\omega - 1$.
\begin{lemma}[Noise recursions]
\label{LemNoiseRecursions}  
  Suppose that the noise sequence $\{\nepsnoiseit{\titer} \}_{\titer
    \geq 1}$ consists of i.i.d. variables, each with zero mean,
  bounded as $|\nepsnoiseit{\titer}| \leq \hanabou$ almost surely, and
  with variance at most $\sigma^2$.  Then for any sequence of
  stepsizes in the interval $(0,1)$ and satisfying the
  bound~\eqref{EqnStepInequal}, we have
  \begin{align}
\label{EqnNoiseRecursions}    
\log \Exs \big[ e^{\slam \Vio_\titer} \big] & \leq \frac{\slam^2
  \sigma^2 \stepit{\titer-1}}{1 - \hanabou \stepit{\titer-1} |\slam|}
\qquad \mbox{for any $|\slam| < \frac{1}{\hanabou
    \stepit{\titer-1}}$.}
\end{align}
\end{lemma}
\begin{proof}
Given the assumed boundedness of $\nepsnoiseit{\titer}$, a standard
argument (see Chapter 2 in the book~\cite{Wai19}) yields that
\begin{align}
  \label{EqnBasic}
  \log \Exs \big[ e^{t \nepsnoiseit{1}} \big] & \leq \frac{t^2
    \sigma^2}{1 - \hanabou |t| } \qquad \mbox{for all $|t| <
    1/\hanabou$.}
\end{align}
We use this bound repeatedly in the argument.

We prove the claim~\eqref{EqnNoiseRecursions} via induction on the
index $\titer$.  The statement is vacuous for $\titer = 1$, since
$\Vio_1 = 0$.  We have $\Vio_2 = \stepit{1} \nepsnoiseit{1}$.  By the
assumptions on $\nepsnoiseit{1}$, for any $|\slam| <
\frac{1}{\stepit{1} \hanabou}$, we have
\begin{align*}
    \log \Exs \big[ e^{\slam \Vio_{2}} \big] = \log \Exs \big[
      e^{\slam \stepit{1} \nepsnoiseit{1}} \big] &
    \stackrel{(i)}{\leq} \frac{\slam^2 \sigma^2 \stepit{1}^2}{1 -
      \hanabou \stepit{1} |\slam|} \; \stackrel{(ii)}{\leq} \;
    \frac{\slam^2 \sigma^2 \stepit{1}}{1 - \hanabou \stepit{1}
      |\slam|},
  \end{align*}
  where step (i) uses the bound~\eqref{EqnBasic} with $t = \slam
  \stepit{1}$; and step (ii) follows since $\stepit{1} \leq 1$.

  We now assume that the claim holds at iteration $\titer$, and then
  verify that it holds at iteration $\titer + 1$.  We have
  \begin{align*}
\log \Exs[e^{\slam \Vio_{\titer + 1}}] & \stackrel{(i)}{=}\log
\Exs[e^{\slam (1-\stepit{\titer}) \Vio_\titer}] + \log \Exs
\big[e^{\slam \stepit{\titer} \nepsnoiseit{\titer}} \big]
\stackrel{(ii)}{\leq} \frac{\slam^2 (1-\stepit{\titer})^2 \sigma^2
  \stepit{\titer-1}}{1 - |\slam| (1-\stepit{\titer}) \stepit{\titer-1}
  \hanabou} + \frac{\slam^2 \sigma^2 \stepit{\titer}^2 }{1 - |\slam|
  \stepit{\titer} \hanabou}
  \end{align*}
  where equality (i) follows from the independence of $\Vio_\titer$
  and $\nepsnoiseit{\titer}$; and inequality (ii) uses the
  bound~\eqref{EqnBasic} as well as the induction hypothesis, and
  holds for all $|\slam| \leq \min \left \{
  \frac{1}{\stepit{\titer}\hanabou}, \frac{1}{ (1-\stepit{\titer})
    \stepit{\titer-1} \hanabou } \right\} = \frac{1}{\stepit{\titer}
    \hanabou}$.  Once again using the assumed
  bound~\eqref{EqnStepInequal} on the stepsizes, we have
\begin{align*}
\log \Exs[e^{\slam \Vio_{\titer + 1}}] & \leq (1 - \stepit{\titer})
\frac{\slam^2 \sigma^2 \stepit{\titer}}{1 - |\slam| \stepit{\titer}
  \hanabou} + \stepit{\titer} \frac{\slam^2 \sigma^2 \stepit{\titer}
}{1 - |\slam| \stepit{\titer} \hanabou} \;= \; \frac{\slam^2 \sigma^2
  \stepit{\titer}}{1 - |\slam| \stepit{\titer} \hanabou},
  \end{align*}
This inequality is valid for any $|\slam| < \frac{1}{\hanabou
  \stepit{\titer}}$, which establishes the
claim~\eqref{EqnNoiseRecursions}.
\end{proof}


\subsubsection{Controlling the expected values of $\|\pathit{\titer}\|_\infty$}

We now state and prove a lemma that allows us to control the expected
values of the $\ell_\infty$-norms of the random sequence
$\{\pathit{\titer} \}_{\titer \geq 1}$ defined via the
recursion~\eqref{EqnDefnPathit}.  The proof of this lemma makes use of
Lemma~\ref{LemNoiseRecursions}.

\begin{lemma}[Bounds on expected values]
\label{LemExpectedBound}  
Consider the sequence $\{\pathit{\titer} \}_{\titer \geq 1}$ generated
by some sequence of stepsizes in the interval $(0,1)$ and satisfying
the bound~\eqref{EqnStepInequal}. Then there is a universal constant
$\unicon$ such that
\begin{align}
\label{EqnExpectedBound}
    \Exs[\|\pathit{\titer}\|_\infty] & \leq \unicon \left \{
    \sqrt{\stepit{\titer}} \| \sigma(\thetastar)\|_\infty \sqrt{\log
      (2 D)} + \stepit{\titer} \spannorm{\thetastar} \log(2 D) \right
    \}.
  \end{align}
\end{lemma}
\begin{proof}
For a given $\slam \in \real$, we have $\Exs \Big[ e^{\slam
    \|\pathit{\titer} \|_\infty} \Big] \leq \sum_{(\state, \action)
  \in \StateSpace \times \ActionSpace} \Exs \Big[ e^{\slam
    \pathit{\titer}(\state, \action)} \Big]$.  Now each random
variable $\pathit{\titer}(\state, \action)$ is an autoregressive
sequence of the type~\eqref{EqnAutoregression}, in which the
underlying noise variables are bounded by $\discount
\spannorm{\thetastar}$, and have variance at most
$\|\myqvarsq{\thetastar}\|_\infty$.  Consequently, from the result of
Lemma~\ref{LemNoiseRecursions}, we have
\begin{align*}
\max_{(\state, \action) \in \StateSpace \times \ActionSpace} \Exs
\left[ e^{\slam \pathit{\titer}(\state, \action)} \right] & \leq \exp
\left( \frac{\slam^2 \|\myqvarsq{\thetastar}\|_\infty
  \stepit{\titer}}{1 - \discount \spannorm{\thetastar} \stepit{\titer}
  |\slam|} \right),
\end{align*}
  and hence
  \begin{align*}    
\Exs \Big[ e^{\slam \|\pathit{\titer} \|_\infty} \Big] & \leq
\underbrace{|\StateSpace| \times |\ActionSpace|}_{ = : \; D} \; \exp
\left( \frac{\slam^2 \|\myqvarsq{\thetastar}\|_\infty
  \stepit{\titer}}{1 - \discount \spannorm{\thetastar} \stepit{\titer}
  |\slam|} \right).
  \end{align*}
Since $\|\pathit{\titer}\|_\infty$ is a non-negative random variable,
the result of Exercise 2.8~(a) in Wainwright~\cite{Wai19} can be
applied, and it yields the claimed bound~\eqref{EqnExpectedBound}.
\end{proof}

\subsection{Proof of Corollary~\ref{CorQlearnLinear}}
\label{AppProofCorQlearnLinear}

Substituting the rescaled linear stepsize choice in the
bound~\eqref{EqnDetClaim} from Corollary~\ref{CorDetClaim} and then
taking expectations, we find that
\begin{align}
  \label{EqnEspresso}
  \Exs \Big[ \|\thetait{\titer+1} - \thetastar \|_\infty \Big] & \leq
  \frac{\|\thetait{1} - \thetastar\|_\infty}{1 + (1-\discount) \titer}
  + \frac{\discount}{1 + (1-\discount) \titer} \sum_{\iter=1}^\titer
  \Exs[\|\pathit{\iter}\|_\infty] +
  \Exs[\|\pathit{\titer+1}\|_\infty].
\end{align}
Next we make use of the bound~\eqref{EqnExpectedBound} from
Lemma~\ref{LemExpectedBound} to control the expected values in
equation~\eqref{EqnEspresso}. Doing so yields
\begin{align*}
  \Exs \Big[ \|\thetait{\titer+1} - \thetastar \|_\infty \Big] & \leq
  \frac{\|\thetait{1} - \thetastar\|_\infty}{1 + (1-\discount) \titer}
  + \frac{\unicon \|\myqvar{\thetastar}\|_\infty \sqrt{\log (2 D)}}{1
    + (1-\discount) \titer} \Term_\titer + \frac{\unicon \,
    \spannorm{\thetastar} \log(2D)}{1 + (1-\discount) \titer}
  \TermB_\titer,
\end{align*}
where $\Term_\titer \defn \frac{1}{\sqrt{\stepit{\titer}}} +
\sum_{\iter=1}^\titer \sqrt{\stepit{\iter}}$ and $\TermB_{\titer}
\defn 1 + \sum_{\iter=1}^\titer \stepit{\iter}$.

In order to complete the proof, it suffices to show that there is a
universal constant $\unicontwo$ such that
\begin{align*}
  \Term_\titer \leq \frac{\unicontwo \sqrt{1 + (1-\discount)
      \titer}}{1 - \discount} \quad \mbox{and} \quad \TermB_\titer
  \leq \frac{\unicontwo \log \big( e + e (1-\discount) \titer \big)}{1 -
    \discount}.
\end{align*}
Beginning with the bound on $\Term_\titer$ and recalling our definition of the
rescaled linear stepsizes, we have
\begin{align*}
\sum_{\iter=1}^\titer \sqrt{\stepit{\iter}} = \sum_{\iter = 1}^\titer
\frac{1}{\sqrt{1 + (1-\discount) \iter}} & \leq \int_1^\titer
\frac{1}{\sqrt{1 + (1-\discount) s}} ds \; \leq \; \frac{2 \sqrt{1 +
    (1-\discount) \titer}}{1-\discount}.
\end{align*}
Combining with the additional $\frac{1}{\sqrt{\stepit{\titer}}}$ term
  yields the claimed bound on $\Term_\titer$.

Turning to the bound on $\TermB_\titer$, we have
\begin{align*}
\TermB_\titer \; = \; 1 + \sum_{\iter = 1}^\titer \frac{1}{1 +
  (1-\discount) \iter} \; \leq \; 1 + \int_1^\titer \frac{1}{1 +
  (1-\discount) s} ds & \leq 1 + \frac{\log(1 + (1-\discount)
  \titer)}{1-\discount} \\
& \leq \frac{2 \log(e +
    (1-\discount) \titer)}{1-\discount},
\end{align*}
which establishes the claim.


\subsection{Proof of Corollary~\ref{CorQlearnPoly}}
\label{AppProofCorQlearnPoly}

We now turn to the proof of our corollary on $Q$-learning with
polynomial stepsizes $\stepit{\titer} = \titer^{-\omega}$.  We require
an auxiliary lemma on exponentially-weighted sums:
\begin{lemma}[Bounds on exponential-weighted sums]
  \label{LemExpBound}
There is a universal constant $\unicon$ such that for all
\mbox{$\omega \in (0, 1)$} and for all \mbox{$\titer \geq \big(\frac{3
    \omega}{ 2(1 - \discount)})^{\frac{1}{1-\omega}}$,} we have
\begin{subequations}
  \begin{align}
 \label{EqnAnnoyA}
e^{- \frac{1-\discount}{1-\omega} \titer^{1-\omega}}
\sum_{\iter=1}^\titer \frac{e^{\frac{1-\discount}{1-\omega}
    \iter^{1-\omega}}}{\iter^{3 \omega/2}} & \leq \unicon \left \{
\frac{e^{-\frac{1-\discount}{1-\omega} \big(\titer^{1-\omega} - 1
    \big)}}{(1-\discount)^{\frac{1}{1-\omega}}} +
\frac{1}{(1-\discount)} \frac{1}{\titer^{\omega/2}} \right \} \quad
\mbox{and} \\
\label{EqnAnnoyB}
e^{- \frac{1-\discount}{1-\omega} \titer^{1-\omega}}
\sum_{\iter=1}^\titer \frac{e^{\frac{1-\discount}{1-\omega}
    \iter^{1-\omega}}}{\iter^{2\omega}} & \leq \unicon \left \{
\frac{e^{-\frac{1-\discount}{1-\omega} \big(\titer^{1-\omega} - 1
    \big)}}{(1-\discount)^{\frac{1}{1-\omega}}} +
\frac{1}{(1-\discount)} \; \frac{1}{\titer^{3 \omega/2}} \right \}.
\end{align}
\end{subequations}
\end{lemma}
\noindent We return to prove this claim in
Appendix~\ref{AppLemExpBound}.

Taking Lemma~\ref{LemExpBound} as given, we first take expectations
over the noise in the bound~\eqref{EqnDetClaimPoly} from
Corollary~\ref{CorDetClaimPoly}.  Using Lemma~\ref{LemExpectedBound}
to control the expected values yields
\begin{align*}
\Exs \left[ \idnorm{\thetait{\titer+1} - \thetastar} \right] & \leq
e^{- \frac{1 - \discount}{1-\omega} (\titer^{1 - \omega} - 1)}
\idnorm{\thetait{1} - \thetastar} + \sigma(\thetastar)\|_\infty
\sqrt{\log (2 D)} T_1 + \spannorm{\thetastar} \log(2 D) T_2,
\end{align*}
where
\begin{align*}
T_1 \defn 1 + e^{- \frac{1-\discount}{1-\omega} \titer^{1-\omega}}
\sum_{\iter=1}^\titer \frac{e^{ \frac{1-\discount}{1-\omega}
    \iter^{1-\omega}}}{\iter^{3 \omega/2}}, \quad \mbox{and} \quad T_2
\defn 1 + e^{- \frac{1-\discount}{1-\omega} \titer^{1-\omega}}
\sum_{\iter=1}^\titer \frac{e^{ \frac{1-\discount}{1-\omega}
    \iter^{1-\omega}}}{\iter^{2 \omega}}.
\end{align*}
Applying Lemma~\ref{LemExpBound} to bound $T_1$ and $T_2$ and
performing some algebra completes the proof of
Corollary~\ref{CorQlearnPoly}.


\subsubsection{Proof of Lemma~\ref{LemExpBound}}
\label{AppLemExpBound}
We prove the bound~\eqref{EqnAnnoyA}.  Define the function $f(s) =
\frac{e^{\frac{1-\discount}{1-\omega} s^{1-\omega}}}{\iter^{3
    \omega/2}}$.  By taking derivatives, we find that $f$ is
decreasing on the interval $[0, c^*]$ and increasing for $s > c^*$,
where $c^* = \big(\frac{3 \omega/2}{1-\discount} \big)^{\frac{1}{1 -
    \omega}}$.  Consequently, as long as $\titer \geq c^*$, we have
\begin{align*}
\sum_{\iter=1}^\titer f(s) & \leq c^* f(1) + \int_{c^*}^\titer f(s) ds
\end{align*}
Integrating by parts, we find that
\begin{align*}
\underbrace{\int_{c^*}^\titer f(s) ds}_{I^*} & = \frac{1}{1-\discount}
  \frac{e^{\frac{1-\discount}{1-\omega} s^{1-\omega}}}{s^{\omega/2}}
  \Big |_{c^*}^\titer + \frac{\omega}{2 (1-\discount)}
  \int_{c^*}^\titer \frac{e^{\frac{1-\discount}{1-\omega}}
    s^{1-\omega}}{s^{1 + (\omega/2)}} ds \\
& \leq \frac{1}{1-\discount} \frac{e^{\frac{1-\discount}{1-\omega}
  \titer^{1-\omega}}}{\titer^{\omega/2}} + \frac{\omega}{2
        (1-\discount)} \int_{c^*}^\titer f(s) \frac{1}{s^{1-\omega}}
  ds  \\
  & \leq \frac{1}{1-\discount} \frac{e^{\frac{1-\discount}{1-\omega}
      \titer^{1-\omega}}}{\titer^{\omega/2}} + \frac{\omega}{2
    (1-\discount)} \frac{1}{(c^*)^{1-\omega}}
  \underbrace{\int_{c^*}^\titer f(s) ds}_{I^*},
\end{align*}
where the final inequality uses the fact that $s \mapsto
1/s^{1-\omega}$ is non-negative and decreasing on the interval $[c^*,
  \titer]$.  Substituting in the expression for $c^*$, we find that
\begin{align*}
  I^* & \leq \frac{1}{1-\discount}
  \frac{e^{\frac{1-\discount}{1-\omega}
      \titer^{1-\omega}}}{\titer^{\omega/2}} + \frac{1}{3} I^*,
\end{align*}
which implies that $I^* \leq \frac{3}{2} \frac{1}{1-\discount}
\frac{e^{\frac{1-\discount}{1-\omega}
    \titer^{1-\omega}}}{\titer^{\omega/2}}$.  Putting together the pieces, we have
shown that
\begin{align*}
\sum_{\iter=1}^\titer f(s) & \leq c^* f(1) + \frac{3}{2}
\frac{1}{1-\discount} \frac{e^{\frac{1-\discount}{1-\omega}
    \titer^{1-\omega}}}{\titer^{\omega/2}} \\
& \leq \left (\frac{3}{1-\discount} \right)^{ \frac{1}{1-\omega}}
e^{\frac{1-\discount}{1-\omega}} + \frac{3}{2} \frac{1}{1-\discount}
\frac{e^{\frac{1-\discount}{1-\omega}
    \titer^{1-\omega}}}{\titer^{\omega/2}},
\end{align*}
which establishes the claim of the first bound~\eqref{EqnAnnoyA}.  The
proof of the second bound~\eqref{EqnAnnoyB} is analogous, so that we
omit the details here.


\section{Details of the ``hard'' example}
\label{AppHard}

The only states with non-trivial variances are states $2$ and $3$, for
any action.  For concreteness, let us focus on the state-action pair
$(\state, \action) = (2, L)$.
We have
\begin{align*}
  \myqvarsq{\thetastar}(2, L) & =
  p \left( \frac{1}{1 - p \discount} - p \frac{1}{1- p \discount} \right)^2
+   (1-p) \Big( 0 - p \frac{1}{1 - p
  \discount} \Big)^2  \\
& \leq \frac{1-p}{(1 - p \discount)^2} + \frac{1 - p}{(1 - p
  \discount)^2}.
\end{align*}
Note that $1 - p = \frac{1 - \discount}{3 \discount} \leq \frac{2}{3}
\frac{1}{1- \discount}$ for $\discount \in [1/2, 1]$, and moreover
$\frac{1}{1 - p \discount} = \frac{3}{4} \frac{1}{1- \discount}$, whence
\begin{align*}
  \myqvarsq{\thetastar}(2, L) & \leq 2 \frac{2}{3} \frac{3}{4}
  \frac{1}{1-\discount} \; = \frac{1}{1-\discount},
\end{align*}
which establishes the upper bound.
As for the lower bound, we have
\begin{align*}
\myqvarsq{\thetastar}(2, L) & \geq (1-p) \Big( 0 - p \frac{1}{1 - p
  \discount} \Big)^2 \; = \; (1 - p) p^2 \frac{9}{16}
\frac{1}{(1-\discount)^2} \; = \; \frac{p^2}{3 \discount} \frac{9}{16}
\frac{1}{1 - \discount},
\end{align*}
where we have used the fact that $1 - p = \frac{1-\discount}{3
  \discount}$.  As long as $\discount \geq 1/2$, then $p \geq 1/3$.  In conjunction
with the lower bound $\frac{1}{3 \discount} \geq 1/3$, we find that
\begin{align*}
\myqvarsq{\thetastar}(2, L) & \geq \frac{1}{9} \frac{1}{3}
\frac{9}{16} \; \frac{1}{1-\discount} \; = \; \frac{1}{48} \;
\frac{1}{1-\discount},
\end{align*}
which establishes the lower bound.




\end{document}